\documentclass[a4paper,oneside]{article}

\usepackage{fullpage}
\usepackage{amsmath,amsfonts,amsthm,amssymb,xspace,bm}
\usepackage{verbatim,dsfont,mathtools,algorithm,algorithmic, url, color}
\usepackage[usenames, dvipsnames, table]{xcolor}
\usepackage{booktabs}

\usepackage[left=1.8cm,right=1.8cm,top=2cm,bottom=2cm]{geometry}
\usepackage{multirow}
\usepackage{xfrac}
\usepackage{wrapfig}
\usepackage{color}
\usepackage{epstopdf}
\usepackage{appendix}
\usepackage{epstopdf}
\usepackage{enumitem}
\usepackage{pifont}
\usepackage{authblk}
\usepackage{subfigure}

\usepackage{color}
\usepackage{epstopdf}
\usepackage{appendix}
\usepackage{epstopdf}
\usepackage{enumitem}
\usepackage{xspace}
\usepackage{xr}

\usepackage{authblk}

\theoremstyle{plain}
\newtheorem{theorem}{Theorem}[section]
\newtheorem{lemma}[theorem]{Lemma}
\newtheorem*{lemma*}{Lemma}

\theoremstyle{definition}

\newtheorem{assumption}{Assumption}

\newcommand{\R}{\mathbb{R}}
\def\w{\mathbf{w}}
\def\y{\mathbf{y}}
\def\x{\mathbf{x}}
\def\dim{p}   
\newcommand{\numsam}{n}
\newcommand{\SGD}{\textsc{SGD}\xspace}
\newcommand{\algo}{\textsc{CheapSVRG}\xspace}

\newcommand{\eqdef}{\stackrel{\textrm{def}}{=}}

\newcommand{\norm}[1]{\left \Vert #1\right \Vert}

\newcommand{\gradf}{\nabla f}

\newcommand{\Q}{Q}

\def\Q{Q}

\def\S{\mathcal{S}}

\def\Q{\mathcal{Q}}

\def\v{\mathbf{v}}

\def\y{\mathbf{y}}
\newcommand{\Ex}[2]{\mathbb{E}_{#1}\left[#2\right]}
\newcommand{\E}[1]{\mathbb{E}\left[#1\right]}

\newcommand{\wt}{\widetilde{\w}}
\newcommand{\mt}{\widetilde{\boldsymbol{\mu}}}

\newcommand{\ws}{\w^\star}

\usepackage{placeins}

\title{Trading-off variance and complexity in stochastic gradient descent}
\author[1]{Vatsal Shah\thanks{vatsalshah1106@utexas.edu}}
\author[1]{Megasthenis Asteris\thanks{megas@utexas.edu}}
\author[1]{Anastasios Kyrillidis\thanks{anastasios@utexas.edu}}
\author[1]{Sujay Sanghavi\thanks{sanghavi@mail.utexas.edu}}
\affil[1]{The University of Texas at Austin}
 
\begin{document}

\maketitle

\begin{abstract}
Stochastic gradient descent (\SGD) is the method of choice for large-scale machine learning problems, 
by virtue of its light complexity per iteration.
However, it lags behind its non-stochastic counterparts with respect to the convergence rate,
due to high variance introduced by the stochastic updates.
The popular Stochastic Variance-Reduced Gradient (\textsc{Svrg}) method mitigates this shortcoming,
introducing a new update rule which requires infrequent 
passes over the entire input dataset to compute the full-gradient.

In this work, we propose \algo,
a stochastic variance-reduction optimization scheme.
Our algorithm is similar to \textsc{Svrg},
but instead of the full gradient,
it uses a surrogate which can be efficiently computed on a small subset of the input data.
It achieves a linear convergence rate 
-- up to some error level, depending on the nature of the optimization problem --
 and features a trade-off between the computational complexity and the convergence rate.
Empirical evaluation shows that \algo performs at least competitively compared to the state of the art.
\end{abstract}



%
%
\section{Introduction}
Several machine learning and optimization problems involve the minimization of a smooth, convex and \emph{separable} cost function $F: \R^\dim \rightarrow \R$:
\begin{align}
	\min_{\w \in \R^\dim} F(\w) := \frac{1}{\numsam} \sum_{i = 1}^\numsam f_i(\w),
	\label{intro:eq_00}
\end{align}
where the $p$-dimensional variable $\w$ represents model parameters,
and each of the functions $f_i(\cdot)$ depends on a single data point.
Linear regression is such an example:
given points $\lbrace (\mathbf{x}_{i}, y_{i})\rbrace_{i=1}^{n}$ in $\mathbb{R}^{p+1}$,
one seeks $\mathbf{w} \in \mathbb{R}^{p}$ that minimizes
the sum of $f_{i}(\mathbf{w}) = (y_{i}-\mathbf{w}^{\top}\mathbf{x}_{i})^{2}$, $i=1, \hdots, n$.
Training of neural networks~\cite{dean2012large, johnson2013accelerating},
multi-class logistic regression~\cite{johnson2013accelerating, roux2012stochastic},
image classification~\cite{bouchard2015accelerating}, 
matrix factorization~\cite{sa2015global}
and many more tasks in machine learning 
entail an optimization of similar form.

Batch gradient descent schemes can effectively solve small- or moderate-scale instances of~\eqref{intro:eq_00}.
Often though,
the volume of input data 
outgrows our computational capacity,
posing major challenges.
Classic batch optimization methods~\cite{nesterov2004introductory, boyd2004convex} perform several passes over the entire input dataset to compute the full gradient, or even the Hessian\footnote{In this work, we will focus on first-order methods only.
Extensions to higher-order schemes is left for future work.}, in each iteration,
incurring a prohibitive cost for very large problems.

Stochastic optimization methods overcome this hurdle
by computing only a surrogate of the full gradient $\nabla F(\mathbf{w})$, based on a small subset of the input data.
For instance, the popular \SGD~\cite{robbins1951stochastic} scheme in each iteration takes a small step in a direction determined by a single, randomly selected data point.
This imperfect gradient step results in smaller progress per-iteration,
though manyfold in the time it would take for a batch gradient descent method to compute a full gradient~\cite{bottou2010large}.

Nevertheless, the approximate `gradients' 
of stochastic methods introduce variance in the course of the optimization.
Notably, vanilla \SGD~methods can deviate from the optimum, even if the initialization point is the optimum \cite{johnson2013accelerating}.
To ensure convergence,
the learning rate has to decay to zero,
which results to sublinear convergence rates~\cite{robbins1951stochastic},
a significant degradation from the linear rate achieved by batch gradient methods.

A recent line of work~\cite{roux2012stochastic, johnson2013accelerating, konevcny2013semi, konevcny2014ms2gd} has made promising steps towards the middle ground of these two extremes.
A full gradient computation is occasionally interleaved with the inexpensive steps of \SGD, dividing the course of the optimization in \emph{epochs}.
Within an epoch, descent directions are formed as a linear combination 
of an approximate gradient (as in vanilla \SGD) and a full gradient vector computed at the beginning of the epoch.
Though not always up-to-date, the full gradient information 
reduces the variance of gradient estimates and provably speeds up the convergence. 


Yet, as the size of the problem grows, 
even an infrequent computation of the full gradient may severely impede the progress of these variance-reduction approaches.
For instance, when training large neural networks~\cite{ngiam2011optimization, sutskever2013importance, dean2012large}),
the volume of the input data rules out the possibility of computing a full gradient within any reasonable time window.
Moreover, in a distributed setting, 
accessing the entire dataset may incur significant tail latencies \cite{zinkevich2010parallelized}.
On the other hand, 
traditional stochastic methods exhibit low convergence rates and in practice frequently fail to come close to the optimal solution in reasonable amount of time.

\paragraph{Contributions.} The above motivate to design algorithms that try to compromise the two extremes
$(i)$ circumventing the costly computation of the full gradient, 
while $(ii)$ admitting favorable convergence rate guarantees.
In this work, we reconsider the computational resource allocation problem in stochastic variance-reduction schemes: 
\emph{given a limited budget of atomic gradient computations, how can we utilize those resources
in the course of the optimization to achieve faster convergence?}
Our contributions can be summarized as follows:
\begin{itemize}
\item [$(i)$]
We propose \algo, a variant of the popular~\textsc{Svrg} scheme \cite{johnson2013accelerating}.
Similarly to \textsc{Svrg},
our algorithm divides time into epochs, but at the beginning of each epoch computes only a surrogate of the full gradient using a subset of the input data.
Then, it computes a sequence of estimates using a modified version of \SGD steps.
Overall, \algo can be seen as a family of stochastic optimization schemes
encompassing \textsc{Svrg} and vanilla \SGD.
It exposes a set of tuning knobs that control trade-offs between the per-iteration computational complexity and the convergence rate.
\item [$(ii)$]
Our theoretical analysis shows
that \algo achieves linear convergence rate 
in expectation and up to a constant factor, that depends on the problem at hand. 
Our analysis is along the lines of similar results for both deterministic and stochastic schemes~\cite{schmidt2014convergence, nedic2001convergence}.
\item [$(iii)$]
We supplement our theoretical analysis with experiments on synthetic and real data. 
Empirical evaluation supports our claims for linear convergence
and shows that 
\algo performs at least competitively with the state of the art.
\end{itemize} 

%


%
%
%
%
%
%
\section{Related work}

There is extensive literature on classic \SGD approaches.
We refer the reader to~\cite{bottou2010large, bertsekas2011incremental} and references therein for useful pointers.
Here, we focus on works related to variance reduction using gradients,
and consider only primal methods; see \cite{shalev2013stochastic, defazio2014finito, mairal2013optimization} for dual.

Roux et al. in~\cite{roux2012stochastic} are among the first that considered variance reduction methods in stochastic optimization.
Their proposed scheme, \textsc{Sag}, achieves linear convergence under smoothness and strong convexity assumptions
and is computationally efficient: it performs only one atomic gradient calculation per iteration. 
However, it is not memory efficient\footnote{The authors show how to reduce memory requirements in the case where $f_i$ depends on a linear combination of $\w$.} as it requires storing all intermediate atomic gradients to generate approximations of the full gradient and, ultimately, achieve variance reduction.

In~\cite{johnson2013accelerating},
Johnson and Zhang improve upon~\cite{roux2012stochastic} with their Stochastic Variance-Reduced Gradient (\textsc{Svrg}) method, which both achieves linear convergence rates and does not require the storage of the full history of atomic gradients. However, \textsc{Svrg} requires a full gradient computation per epoch.
The \textsc{S2gd} method of~\cite{konevcny2013semi} follows similar steps with \textsc{Svrg},
with the main difference lying in the number of iterations within each epoch, which is chosen according to a specific geometric law.
Both~\cite{johnson2013accelerating} and~\cite{konevcny2013semi} rely on the assumptions that $F(\cdot)$ is strongly convex and $f_i(\cdot)$'s are smooth.

Recently, Defazio et al. propose \textsc{Saga} \cite{defazio2014saga}, a fast incremental gradient method in the spirit of \textsc{Sag} and \textsc{Svrg}. 
\textsc{Saga} works for both strongly and plain convex objective functions, as well as in proximal settings.
However, similarly to its predecessor~\cite{roux2012stochastic}, it does not admit low storage cost.

While writing this paper, we also became aware of the independent work of \cite{frostig2014competing} and \cite{harikandeh2015stopwasting}, where similar ideas are developed. In the first case, the authors consider a streaming version of \textsc{Svrg} algorithm and, the purpose of that research is different from the present work, as well as the theoretical analysis followed. The results presented in the latter work are of the same flavor with ours. 


Finally, we note that proximal~\cite{konevcny2014ms2gd, defazio2014saga, allen2015univr, xiao2014proximal} and distributed \cite{reddi2015variance, lee2015distributed, zhang2015fast} variants have also been proposed for such stochastic settings.
We leave these variations out of comparison and consider similar extensions to our approach as future work. 

%
%
\section{Our variance reduction scheme}

We consider the minimization in~\eqref{intro:eq_00}.
In the $k$th iteration, vanilla \SGD generates a new estimate
\begin{align*}
	\w_{k} =  \w_{k-1}  - \eta_{k} \cdot \nabla f_{i_{k}} (\w_{k-1}),
\end{align*}
based on the previous estimate $\w_{k-1}$
and
the atomic gradient of a component $f_{i_{k}}$, 
where index ${i_k}$ is selected uniformly at random from $\left\{1, \dots, \numsam\right\}$.
The intuition behind \SGD is that 
in expectation
its update direction aligns with the gradient descent update. 
But, contrary to gradient descent, \SGD is not guaranteed to move towards the optimum in each single iteration.
To guarantee convergence, it employs a decaying sequence of step sizes $\eta_k$, which in turn impacts the rate at which convergence occurs.

\textsc{Svrg}~\cite{johnson2013accelerating}
alleviates the need for decreasing step size by 
dividing time into epochs and 
interleaving a computation of the full gradient between consecutive epochs.
The full gradient information 
 $
 	{\widetilde{\boldsymbol{\mu}}
 	=
 	\sfrac{1}{\numsam} \sum_{i = 1}^{\numsam} \nabla f_i(\wt_{t})},
 $
 where $\wt_{t}$ is the estimate available at the beginning of the $t$th epoch, is used to steer the subsequent steps 
 and counterbalance the variance introduced by the randomness of the stochastic updates.
 Within the $t$th epoch, 
 \textsc{Svrg} computes a sequence of estimates
 $
 	\w_{k} =  \w_{k-1}  - \eta \cdot \v_k,
$
where ${\w_0 = \wt_{t}}$,
and
\begin{align*}
	\v_k
	=
	\nabla f_{i_{k}}(\w_{k-1})
	- \nabla f_{i_{k}} (\widetilde{\w}) 
	+ \widetilde{\boldsymbol{\mu}}
\end{align*}
is a linear combination of full and atomic gradient information. 
Based on this sequence, it computes the next estimate $\wt_{t+1}$, which is passed down to the next epoch.
Note that $\v_k$ is an \emph{unbiased} estimator of the gradient $\nabla F(\w_{k-1})$, \textit{i.e.}, $\Ex{i_k}{\v_k} = \nabla F(\w_{k-1})$.

As the number of components $f_{i}(\cdot)$ grows large,
the computation of the full gradient~$\mt$, at the beginning of each epoch, becomes a computational bottleneck.
A natural alternative is to compute a surrogate $\mt_{\S}$, using only a small subset $\S \subset [\numsam]$ of the input data. 

\paragraph{Our scheme.}
We propose \algo,
a variance-reduction stochastic optimization scheme.
Our algorithm can be seen as a unifying scheme of existing stochastic methods including \textsc{Svrg} and vanilla \SGD.
Its steps are outlined in Alg.~\ref{algo:beaverSGD}.

\begin{algorithm}[!t]
	\caption{\algo}
	\label{algo:beaverSGD}
	{
	\begin{algorithmic}[1]
	   	\STATE \textbf{Input}: $\widetilde{\w}_0, \eta, s, K, T$.
	   	\STATE \textbf{Output}: $\wt_T$.
	   	\FOR{ $t = 1, 2, \dots, T$ }
	   		\STATE Randomly select $\S_{t} \subset [\numsam]$ with cardinality $s$.
	   		\STATE Set $\widetilde{\w} = \widetilde{\w}_{t-1}$ and $\S = \S_{t}$.
				\STATE $\widetilde{\boldsymbol{\mu}}_\S = \sfrac{1}{s} \sum_{i \in \S} \gradf_i(\wt)$.
				\STATE $\w_0 = \widetilde{\w}$.
				\FOR{ $k = 1, \dots, K-1$ }
					\STATE Randomly select $i_k \subset [\numsam]$.
					\STATE $\v_k = \nabla f_{i_k}(\w_{k-1}) - \nabla f_{i_k} (\widetilde{\w}) + \widetilde{\boldsymbol{\mu}}_\S$.
					\STATE $\w_{k} = \w_{k-1} - \eta \cdot \v_k$.
				\ENDFOR
				\STATE $\widetilde{\w}_t = \frac{1}{K} \sum_{j = 0}^{K-1} \w_{j}$.
			\ENDFOR
	\end{algorithmic}
	}
\end{algorithm}

\algo divides time into epochs.
The~\mbox{$t$th} epoch begins at an estimate $\wt=\widetilde{\w}_{t-1}$, inherited from the previous epoch.
For the first epoch, that estimate is given as input, $\widetilde{\w}_0 \in \mathbb{R}^{p}$.
The algorithm selects a set $\S_{t} \subseteq [\numsam]$ uniformly at random, with cardinality~$s$, for some parameter $0 \le s \le n$.
Using only the components of $F(\cdot)$ indexed by $\S$, it computes
\begin{align}
\widetilde{\boldsymbol{\mu}}_\S
\eqdef
\frac{1}{s} \sum_{i \in \S} \gradf_i(\wt),
\label{eq:mt_S}
\end{align} 
a surrogate of the full-gradient $\mt$.

Within the $t$th epoch, the algorithm generates a sequence of $K$ estimates $\w_{k}$, $k=1,\hdots, K$,  
through an equal number of \SGD-like iterations, using a modified, `biased' update rule.
Similarly to \textsc{Svrg},
starting from $\w_{0} = \wt$,
in the $k$th iteration, it computes
$$
\w_{k} = \w_{k-1} - \eta \cdot \v_k,
$$
where
$\eta > 0$ is a constant step-size
and
\begin{align*}
	\v_k = \nabla f_{i_{k}}(\w_{k-1}) - \nabla f_{i_{k}} (\widetilde{\w}) + \widetilde{\boldsymbol{\mu}}_\S.
\end{align*}
The index $i_k$ is selected uniformly at random
from $[\numsam]$, independently across iterations.\footnote{In the Appendix, we also consider the case where the inner loop uses a mini-batch $\Q_k$ instead of a single component $i_{k}$. The cardinality $q=|\Q_k|$ is a user parameter.}
The estimates obtained from the iterations of the inner loop (lines $8$-$12$),
are averaged to yield the estimate $\wt_{t}$ of the current epoch, and is used to initialize the next.

Note that during this \SGD phase,
the index set $\S$ is fixed. 
Taking the expectation w.r.t. index $i_{k}$,
we have
\begin{align*}
	\Ex{i_k}{\v_k} = \nabla F(\w_{k-1}) - \nabla F(\wt) + \mt_\S.
\end{align*}
Unless ${\S = [\numsam]}$, the update direction $\v_{k}$ is a biased estimator of $\nabla F(\w_{k-1})$.
This is a key difference from the update direction used by \textsc{Svrg} in~\cite{johnson2013accelerating}.
Of course, since $\S$ is selected uniformly at random in each epoch, then across epochs 
$
	\Ex{\S}{\mt_\S} = \nabla F(\wt),
$
where the expectation is with respect to the random choice of~$\S$.
Hence, on expectation, the update direction $\v_{k}$ can be considered an unbiased surrogate of $\nabla{F(\w_{k-1})}$.

Our algorithm can be seen as a unifying framework,
encompassing existing stochastic optimization methods.
If the tunning parameter $s$ is set equal to $0$,
the algorithm reduces to vanilla \SGD,
while for $s=n$, we recover \textsc{Svrg}.
Intuitively, $s$ establishes a trade-off between the quality of the full-gradient surrogate generated at the beginning of each epoch and the associated computational cost.
\section{Convergence analysis}{\label{sec:analysis}}


In this section, we provide a theoretical analysis of our algorithm under standard assumptions,
along the lines of~\cite{schmidt2014convergence, nedic2001convergence}.
We begin by defining those assumptions and the notation used in the remainder of this section.

\paragraph{Notation.}
We use ${[n]}$ to denote the set $\left\{1, \dots, {n}\right\}$.
For an index $i$ in $[\numsam]$,
$\nabla f_{i}(\w)$ denotes the atomic gradient on the $i$th component $f_{i}$.
We use $\Ex{i}{\cdot}$ to denote the expectation with respect the random variable~$i$.
With a slight abuse of notation,
we use $\Ex{[i]}{\cdot}$ to denote the expectation with respect to $i_1, \dots, i_{K-1}$.

\paragraph{Assumptions.}
Our analysis is based on the following assumptions,
which are common across several works in the
stochastic optimization literature.
\begin{assumption}[Lipschitz continuity of $\nabla f_i$]
{\label{ass:00}}
Each~$f_i$ in~\eqref{intro:eq_00} has $L$-Lipschitz continuous gradients, \textit{i.e.},
there exists a constant $L>0$ such that for any $\w, \w^{\prime} \in \R^\dim$,
\begin{align*}
	f_i(\w) \leq f_i(\w') + \gradf_i(\w')^\top (\w-\w') + \tfrac{L}{2} \norm{\w-\w'}^2_2.
\end{align*}
\end{assumption}

\begin{assumption}[Strong convexity of $F$]
\label{ass:01}
The function
${F(\w) = \sfrac{1}{\numsam} \sum_{i = 1}^{\numsam} f_i(\w)}$
is $\gamma$-strongly convex for some constant $\gamma > 0$, \textit{i.e.}, for any $\w, \w^{\prime} \in \R^\dim$,
\begin{align*}
F(\w) -  F(\w') - \nabla F(\w')^\top (\w-\w') \geq \tfrac{\gamma}{2} \norm{\w-\w'}^2_2.
\end{align*}
\end{assumption}

\begin{assumption}[Component-wise bounded gradient]
\label{ass:02}
There exists ${\xi > 0}$ such that
$\|\nabla f_i(\w)\|_2 \leq \xi$, $\forall \w$ in the domain of $f_i$,
for all $i \in [n]$. 
\end{assumption}
Observe that Asm.~\ref{ass:02} is satisfied if the components $f_{i}(\cdot)$ are $\xi$-Lipschitz functions.
Alternatively, Asm.~\ref{ass:02} is satisfied when $F(\cdot)$ is $\xi'$-Lipschitz function and $\max_i\left\{\|\nabla f_i(\w)\|_2 \right\} \leq C \cdot \|\nabla F(\w)\|_2 \leq C \cdot \xi' =: \xi$. This is known as the \emph{strong growth condition} \cite{schmidt2013fast}.\footnote{This condition is rarely satisfied in many practical cases.
However, similar assumptions have been used to show convergence of Gauss-Newton-based schemes \cite{bertsekas1999nonlinear}, as well as deterministic incremental gradient methods \cite{solodov1998incremental, tseng1998incremental}.}

\begin{assumption}[Bounded Updates]
\label{ass:03}
For each of the estimates $\w_k, ~k = 0, \dots, K-1$, 
we assume that the expected distance $\mathbb{E}\left[\|\w_k - \ws\|_2\right]$ is upper bounded by a constant.
Equivalently, there exists $\zeta > 0$ such that
\begin{align*}
	\sum_{j = 0}^{K-1} \Ex{[i]}{\|\w_j - \ws\|_2} \leq \zeta.
\end{align*}
\end{assumption}
We note that Asm.~\ref{ass:03} is non-standard, but was required for our analysis.
An analysis without this assumption is an interesting open problem.

\subsection{Guarantees}
We show that, under Asm.~\ref{ass:00}-\ref{ass:03}, the algorithm will converge --in expectation-- with respect to the objective value, achieving a linear rate, up to a constant neighborhood of the optimal, depending on the configuration parameters and the problem at hand.
Similar results have been reported for \SGD~\cite{schmidt2014convergence}, as well as deterministic incremental gradient methods \cite{nedic2001convergence}.

\begin{theorem}[Convergence]
\label{theorem:convergence}
Let $\ws$ be the optimal solution for minimization~\eqref{intro:eq_00}.
Further, let $s$, $\eta$, $T$ and $K$ be user defined parameters such that
\begin{align*}
	\rho \eqdef 
	\frac{1}{\eta \cdot \left(1 - {4 L \cdot \eta}\right) \cdot {K \cdot \gamma}} + \frac{4L \cdot \eta \cdot\left(1 + \sfrac{1}{s}\right)}{\left(1 - 4 L \cdot \eta\right)} 
	< 1.
\end{align*}
Under Asm.~\ref{ass:00}-\ref{ass:03}, \algo
outputs $\wt_T$ such that
\begin{align*}
	&
	\mathbb{E}\bigl[ F(\wt_T) - F(\ws) \bigr] \leq \rho^{T} \cdot \left(F(\wt_0)-F(\ws)\right) + \kappa,
\end{align*} where 
$\kappa \eqdef \frac{1}{1 - 4 L \eta} \cdot \left(\frac{2\eta}{s} + \frac{\zeta}{K}\right) \cdot \max\left\{\xi, \xi^2\right\} \cdot \frac{1}{1 - \rho}$.
\end{theorem}

We remark the following:
\begin{itemize}[leftmargin=0.5cm]
\item [$(i)$]
The condition $\rho < 1$ ensures convergence up to a neighborhood around~$\w^\star$.
In turn, we require that
\begin{small}
\begin{align*}
\eta < \dfrac{1}{4L\left((1 + \theta)+ \sfrac{1}{s}\right)} ~\text{and}~ K > \frac{1}{(1 - \theta) \eta \left(1 - 4 L \eta\right) \gamma},
\end{align*}
\end{small}%
for appropriate $\theta \in (0, 1)$.

\item [$(ii)$]
The value of $\rho$ in Thm.~\ref{theorem:convergence} is similar to that of~\cite{johnson2013accelerating}: for sufficiently large $K$, there is a $(1 + \sfrac{1}{s})$-factor deterioration in the convergence rate, due to the parameter~$s$.
We note, however, that our result differs from~\cite{johnson2013accelerating}
in that Thm.~\ref{theorem:convergence} guarantees convergence \emph{up to a neighborhood around~$\w^\star$}.
To achieve the same convergence rate with \cite{johnson2013accelerating}, we require ${\kappa =O(\rho^T)}$,
which in turn implies that $s = \Omega(\numsam)$. To see this, consider a case where the condition number $L$ is constant and $\sfrac{L}{\gamma} = \numsam$.
Based on the above, we need $K = \Omega(n)$.
This further implies that, in order to bound the additive term in Thm.~\ref{theorem:convergence}, $s = \Omega(n)$ is required for $O(\rho^T) \ll 1$.

\item [$(iii)$] When $\xi$ is sufficiently small, Thm.~\ref{theorem:convergence} implies that
$$
\mathbb{E}\bigl[ F(\wt_T) - F(\ws) \bigr] \lesssim \rho^{T} \cdot \left(F(\wt_0)-F(\ws)\right), 
$$
\textit{i.e.}, that even ${s = 1}$ leads to (linear) convergence; 
In Sec.~\ref{sec:experiments}, we empirically show cases where even for ${s = 1}$, our algorithm works well in practice.
\end{itemize}

The following theorem establishes the \emph{analytical complexity} of \algo; the proof is provided in the Appendix.

\begin{theorem}[Complexity]\label{theorem:complexity}
For some accuracy parameter $\epsilon$,
if $\kappa \leq \sfrac{\epsilon}{2}$,
then for 
suitable $\eta$, $K$,
and
\begin{align*}
T
\geq
\left(\log \tfrac{1}{\rho}\right)^{-1} \cdot \log 
\Bigl(
	\tfrac{2\left(F(\wt_0) - F(\ws)\right)}{\epsilon}
\Bigr),
\end{align*} 
Alg.~\ref{algo:beaverSGD} 
outputs $\wt_T$ such that $\E{F(\wt_T) - F(\ws)} \leq \epsilon$.
 Moreover, the total complexity is $O\left( \left(2K + s \right)\log \sfrac{1}{\epsilon}\right)$ atomic gradient computations.
\end{theorem}

\subsection{Proof of Theorem \ref{theorem:convergence}}
We proceed with an analysis of Alg.~\ref{algo:beaverSGD},
and in turn the proof of Thm.~\ref{theorem:convergence}, starting from its core inner loop (Lines $8$-$12$)
and show that in expectation, the steps of the inner loop make progress towards the optimum point.
Then, we move outwards to the `wrapping' loop that defines consecutive epochs.

\textbf{Inner loop.} Fix an epoch, say the $t$th iteration of the outer loop.
Starting from a point $\w_0 \in \mathbb{R}^{p}$ (which is effectively the estimate of the previous epoch),
the inner loop performs $K$ steps, 
using the partial gradient information vector $\mt_\S \in \mathbb{R}^{p}$, for a fixed set $\S$.

Consider the $k$th iteration of the inner loop.
For now, let the estimates generated during previous iterations be known history.
By the update in Line $13$, we have
\begin{align}
	\mathbb{E}_{i_k}\mathopen{}
	\Bigl[
	\left\|
		\w_k -\ws
	\right\|_{2}^{2}
	\Bigr]
	= 
	\left\| \w_{k-1}-\ws \right\|_{2}^{2} 
	- {2\eta \cdot \left(\w_{k-1} - \ws \right)^\top\mathbb{E}_{i_k}\mathopen{}\bigl[ \v_k \bigr]
	+ 
	\eta^2 \cdot \mathbb{E}_{i_k}\mathopen{}
	\bigl[\norm{\v_k}_2^2 \bigr]},
	\label{eq:001}
\end{align} 
where the expectation is with respect to the choice of index~$i_k$. 
We develop an upper bound for the right hand side of~\eqref{eq:001}.
By the definition of $\v_k$ in Line 10, 
\begin{align}
	\mathbb{E}_{i_k}\mathopen{}
	\left[
		\v_{k}
	\right]
	&=
	\mathbb{E}_{i_k}\mathopen{}
	\left[
		\nabla f_{i_k}(\w_{k-1}) - \nabla f_{i_k}(\wt) + \mt_\S
	\right] \nonumber\\
	&=
	\nabla F(\w_{k-1}) - \nabla F(\wt) + \mt_\S,
	\label{eq:expected_v_k}
\end{align}
Similarly,
\begin{align} 
\mathbb{E}_{i_{k}}\mathopen{}
	\Bigl[
		\left\| \v_{k} \right\|_{2}^{2}
	\Bigr]
	&=
	\mathbb{E}_{i_{k}}\mathopen{}
	\left[
		\left\|
			\nabla f_{i_{k}}(\w_{k-1}) - \nabla f_{i_{k}}(\wt) + \mt_\S
		\right\|_{2}^{2}
	\right]
	\nonumber \\ 
	& \stackrel{}{\leq}
	4\cdot 
	\mathbb{E}_{i_{k}}\mathopen{}
	\left[
		\|\nabla f_{i_{k}}(\w_{k-1}) - \nabla f_{i_{k}}(\ws)\|_2^2
	\right] +
		4\cdot
		\mathbb{E}_{i_{k}}\mathopen{}
		\left[
			\left\|\nabla f_{i_{k}}(\wt) - \nabla f_{i_{k}}(\ws)\right\|_{2}^{2}
		\right]
		+2 \cdot \|\mt_\S\|_2^2 
		\nonumber\\
	&\stackrel{}{\leq}
		8L \cdot
	\left(
		F(\w_{k-1}) - F(\ws) + F(\wt) - F(\ws)
	\right) + 2\cdot \|\mt_\S\|_{2}^{2}.
	\label{bound_on_expected_norm_of_v_k}
\end{align}
The first inequality follows from the fact that
$\norm{\x - \y}_2^2 \leq 2\norm{\x}_2^2 + 2\norm{\y}_2^2$,
for any $\mathbf{x}$, $\mathbf{y} \in \mathbb{R}^{p}$.
The second inequality is due to eq. (8) in~\cite{johnson2013accelerating}. 
Continuing from~\eqref{eq:001} and taking into account~\eqref{eq:expected_v_k} and \eqref{bound_on_expected_norm_of_v_k}, we have
\begin{align}
\!\!
	\mathbb{E}_{i_{k}}\mathopen{}
	\left[
	\left\|
		\w_k -\ws
	\right\|_{2}^{2}
	\right]
&\leq \norm{\w_{k-1}-\ws}_2^2 
	- 2\eta \cdot (\w_{k-1} - \ws)^\top \left(\nabla F(\w_{k-1}) - \nabla F(\wt) + \mt_\S\right) \nonumber\\
&\;\;\;\;
	+ 8 L \eta^2 \cdot \left(F(\w_{k-1}) - F(\ws) + F(\wt) - F(\ws)\right)
	+ 2\eta^2 \cdot \|\mt_\S\|_2^2.
	\label{progress_inner_ik}
\end{align}
Inequality~\eqref{progress_inner_ik} establishes an upper bound on the distance of the $k$-th iteration estimate $\w_{k}$ from the optimum, conditioned on the random choices $[i_{k-1}] = \lbrace{ i_1, i_2, \dots, i_{k-1} }\rbrace$ of previous iterations.
Taking expectation over $[i_{k}]$,
for any $k = 1, \dots, K - 1$, we have
\begin{small}\begin{align}
	\mathbb{E}_{[i]}\mathopen{}
	\left[
		\| \w_k -\ws \|_{2}^{2}
	\right] 
	&\leq 
	\mathbb{E}_{[i]}\mathopen{}
	\left[
		\| \w_{k-1}-\ws \|_{2}^{2}
	\right]
	- 2\eta \cdot 
	\mathbb{E}_{[i]}\mathopen{}
	\left[
		(\w_{k-1} - \ws)^\top \left(\nabla F(\w_{k-1}) - \nabla F(\wt) + \mt_\S\right)
	\right] 
	\nonumber\\
	&\;\;\;\;
	+ 8 L \eta^2 \cdot \left(\Ex{[i]}{F(\w_{k-1})} - F(\ws) + F(\wt) - F(\ws)\right) 
	+ 2\eta^2 \cdot \|\mt_\S\|_2^2.
\end{align} \end{small}%
Note that $F(\ws)$, $F(\wt)$, $\nabla F(\wt)$,
$\|\mt_\S\|_2^2$ and, $\|\w_0 - \ws\|_2^2$ are constants w.r.t. $[i]$.
Summing over $k$,
\begin{small}\begin{align}
\Ex{[i]}{\norm{\w_K -\ws}_2^2} &\leq \norm{\w_0-\ws}_{2}^{2} 
- 2\eta \cdot \sum_{j = 0}^{K - 1} \Ex{[i]}{(\w_{j} - \ws)^\top \left(\nabla F(\w_{j}) - \nabla F(\wt) + \mt_\S\right)} \nonumber\\
&\;\;\;\;
+ 8 L \eta^2 \cdot \sum_{j = 0}^{K - 1} \left(\Ex{[i]}{F(\w_j)} - F(\ws)\right) 
+ 8KL \eta^2 \left(F(\wt) - F(\ws)\right) + 2 \eta^2 \cdot K \cdot \|\mt_\S\|_2^2.
\label{main_ineq_K}
\end{align}\end{small}%
For the second term on the right hand side, we have:
\begin{small}\begin{align}
	-\sum_{j = 0}^{K - 1} \Ex{[i]}{(\w_{j} - \ws)^\top \left(\nabla F(\w_{j}) - \nabla F(\wt) + \mt_\S\right)} 
	\leq \sum_{j = 0}^{K-1} \left(\Ex{[i]}{F(\w_j)} - F(\ws)\right)
	+ \sum_{j = 0}^{K - 1} \Ex{[i]}{(\w_{j} - \ws)^\top\left(\nabla F(\wt) - \mt_\S\right)}
	\label{bound_on_second_term}
\end{align}\end{small}%
where the inequality follows from the convexity of $F(\cdot)$. 
Continuing from~\eqref{main_ineq_K},
taking into account~\eqref{bound_on_second_term}
and the fact that
$
\mathbb{E}_{[i]}
\bigl[
	\norm{\w_K -\ws}_{2}^{2}
\bigr]
\geq 0,
$
we obtain
\begin{small}%
\begin{align}
	2\eta\left(1 - 4 L \eta \right) \cdot \sum_{j = 0}^{K - 1} \left(\Ex{[i]}{F(\w_j)} - F(\ws)\right) 	
	&\leq \norm{\w_0-\ws}_2^2 
	+ 2\eta \cdot \sum_{j = 0}^{K - 1} \Ex{[i]}{(\w_{j} - \ws)^\top \left(\nabla F(\wt) - \mt_\S\right)} \nonumber \\
	&\;\;\;\;
	+ 8KL \eta^2 \left(F(\wt) - F(\ws)\right) + 2\eta^2 \cdot K \cdot \|\mt_\S\|_{2}^{2}.
	\label{eq:main_ineq}
\end{align}%
\end{small}%
By the convexity of $F(\cdot)$,
$$
	F\left( \wt_t \right) 
	=
	F\Bigl(
		\frac{1}{K} \sum_{j = 0}^{K-1} \w_j 
	 \Bigr)
	\leq
	\frac{1}{K} \sum_{j = 0}^{K-1} F(\w_j).
$$ 
Also, by the strong convexity of $F(\cdot)$ (Asm.~\ref{ass:01}),
$$
	\|\w_0 - \ws\|_2 
	\leq
	\sfrac{2}{\gamma} \cdot
	\left(F(\w_{0}) - F(\ws)\right).
$$ 
Continuing from~\eqref{eq:main_ineq},
taking into account the above and 
recalling that $\w_{0}=\wt = \wt_{t-1}$,
we obtain
\begin{small}
\begin{align}
B &\eqdef
2\eta\left(1 - 4 L \eta \right) \cdot K \cdot \Ex{[i]}{F(\wt_t) - F(\ws)} \nonumber \\ 
	&\leq
	 \left(\frac{2}{\gamma} + 8KL \eta^2 \right) \left(F(\wt_{t-1}) - F(\ws)\right) 
	+ 2\eta^2 \cdot K \cdot \|\mt_\S\|_2^2
	+ 2\eta \cdot \sum_{j = 0}^{K - 1} \Ex{[i]}{(\w_{j} - \ws)^\top \left(\nabla F(\wt_{t-1}) - \mt_\S\right)}. 
	\label{eq:003}
\end{align}
\end{small}%
The last sum in~\eqref{eq:003} can be further upper bounded:
\begin{align}
\sum_{j = 0}^{K - 1} \Ex{[i]}{(\w_{j} - \ws)^\top \left(\nabla F(\wt_{t-1}) - \mt_\S\right)} 
&\leq
	\sum_{j = 0}^{K - 1} \Ex{[i]}{\|\w_{j} - \ws\|_2 \|\nabla F(\wt_{t-1}) - \mt_\S\|_2}  \nonumber\\
&\leq
	\|\nabla F(\wt_{t-1}) - \mt_\S\|_2 \cdot \sum_{j = 0}^{K - 1} \Ex{[i]}{\|\w_{j} - \ws\|_2} \nonumber\\
&\leq
	\zeta \cdot \|\nabla F(\wt_{t-1}) - \mt_\S\|_2. \nonumber
\end{align} 
The first inequality follows from Cauchy-Schwarz,
the second from the fact that $\|\nabla F(\wt_{t-1}) - \mt_\S\|_2$ is independent from the random variables $[i]$ ($\S$ and $\wt_{t-1}$ are fixed),
while the last one follows from Asm.~\ref{ass:03}.
Incorporating the above upper bound in~\eqref{eq:003},
we obtain
\begin{align}
	B&
	\leq 
	\left(\tfrac{2}{\gamma} + 8KL \eta^2 \right) \left(F(\wt_{t-1}) - F(\ws)\right) 
	+2\eta^2 \cdot K \cdot \|\mt_\S\|_2^2 
	+2\eta \zeta \cdot \|\mt_{\S^c}\|_2,
	\label{eq:004}
\end{align}
where
$
	\mt_{\S^c}
	\eqdef
	F(\wt_{t-1}) - \mt_\S = \frac{1}{\numsam - s} \cdot \sum_{i \in \S^c} \nabla f_i(\wt_{t-1}).
$
The inequality in~\eqref{eq:004} effectively establishes a recursive bound on $\Ex{[i]}{F(\wt_t) - F(\ws)}$ using only the estimate sequence produced by the epochs.

\textbf{Outer Loop.}
Taking expectation over~$\S_{t}$, assuming that $\eta$ is such that $1 - 4L\eta > 0$,
we have
\begin{small}
\begin{align}
	\Ex{[i], \S_{t}}{F(\wt_t) - F(\ws)} &\leq 	
	\frac{\sfrac{2}{\gamma} + 8KL \eta^2}{2\eta\left(1 - 4 L \eta\right) \cdot K} \Ex{[i], \S_{t}}{F(\wt_{t-1}) - F(\ws)} 
	+ \frac{\eta}{\left(1 - 4 L \eta\right)} \cdot \Ex{[i], \S_{t}}{\|\mt_\S\|_2^2} \nonumber \\
	&\;\;\;\;
	+ \frac{\zeta}{\left(1 - 4 L \eta \right) \cdot K} \cdot \Ex{[i], \S_{t}}{\|\mt_{\S^c}\|_2}.
\label{eq:005}
\end{align}
\end{small}%
To further bound the right-hand side, note that:
\begin{align*}
\Ex{[i], \S_{t}}{\norm{\mt_\S}_2^2} &=
\mathbb{E}_{[i], \S_{t}}
\Bigl[
	\bigl\|
		s^{-1} \cdot \sum_{i \in \S} \gradf_i(\wt_{t-1})
	\bigr\|_{2}^{2}
\Bigr] \\
&
\stackrel{(i)}{\leq} 
	\frac{2}{\numsam \cdot s} 
	\left[
		\sum_{i = 1}^\numsam \norm{\gradf_i(\wt_{t-1}) - \gradf_i(\ws)}_2^2 
		+
		\norm{\gradf_i(\ws)}_2^2 
	\right]\\
&
\stackrel{(ii)}{\leq}
\frac{4L}{s} \cdot \left(F(\wt_{t-1}) - F(\ws) \right) + \frac{2}{\numsam\cdot s} \sum_{i = 1}^\numsam \norm{\gradf_i(\ws)}_2^2,
\end{align*}
where $(i)$ is due to $\norm{\x - \y}_2^2 \leq 2\norm{\x}_2^2 + 2\norm{\y}_2^2$ and $(ii)$ due to eq. (8) in \cite{johnson2013accelerating}.
By Asm.~\ref{ass:02}, $\norm{\gradf_i(\ws)}_2 \leq \xi$. 
Using similar reasoning, under Asm.~\ref{ass:02},
$
\Ex{[i], \S_{t}}{\|\mt_{\S^c}\|_2} \leq \xi.
$
Combining the above with \eqref{eq:005}, we get
\begin{align}
\Ex{[i], \S_{t}}{F(\wt_t) - F(\ws)} \leq 
\frac{\tfrac{2}{\gamma} + 8KL \eta^2\left(1 + \tfrac{1}{s}\right)}{2\eta\left(1 - 4 L \eta\right) \cdot K} \cdot 
\mathbb{E}_{[i], \S_{t}}
\Bigl[
	F(\wt_{t-1}) - F(\ws)
\Bigr] 
+\frac{1}{1 - 4 L \eta} \cdot \left(\frac{2\eta}{s} + \frac{\zeta}{K}\right) \cdot \max\left\{\xi, \xi^2\right\}.
\label{eq:006}
\end{align}
Let $[\S] \eqdef \lbrace{ \S_0, \S_1, \dots, S_{T-1}}\rbrace$.
Also, let
\begin{align*}
	\varphi_t
	\eqdef
	\mathbb{E}_{[i], [\S]}\mathopen{}
	\bigl[
		F(\wt_t ) - F(\ws)
	\bigr],
\end{align*}
and $\rho$ as defined in Thm.~\ref{theorem:convergence}.
Taking expectation with respect to $[\S]$, \eqref{eq:006} becomes
\begin{align*}
	\varphi_t \leq \rho \cdot \varphi_{t-1} + \frac{1}{1 - 4 L \eta} \cdot \left(\frac{2\eta}{s} + \frac{\zeta}{K}\right) \cdot \max\left\{\xi, \xi^2\right\}
\end{align*}
Finally, unfolding the above recursion, we obtain
$
	\varphi_T 
	\leq
	\rho^T \cdot \varphi_{0} 
	+ \kappa
$
where $\kappa$ is defined in Thm.~\ref{theorem:convergence}.
This completes the proof of the theorem.

%
%

\begin{figure*}[t!]
\centering
\includegraphics[width=0.32\textwidth]{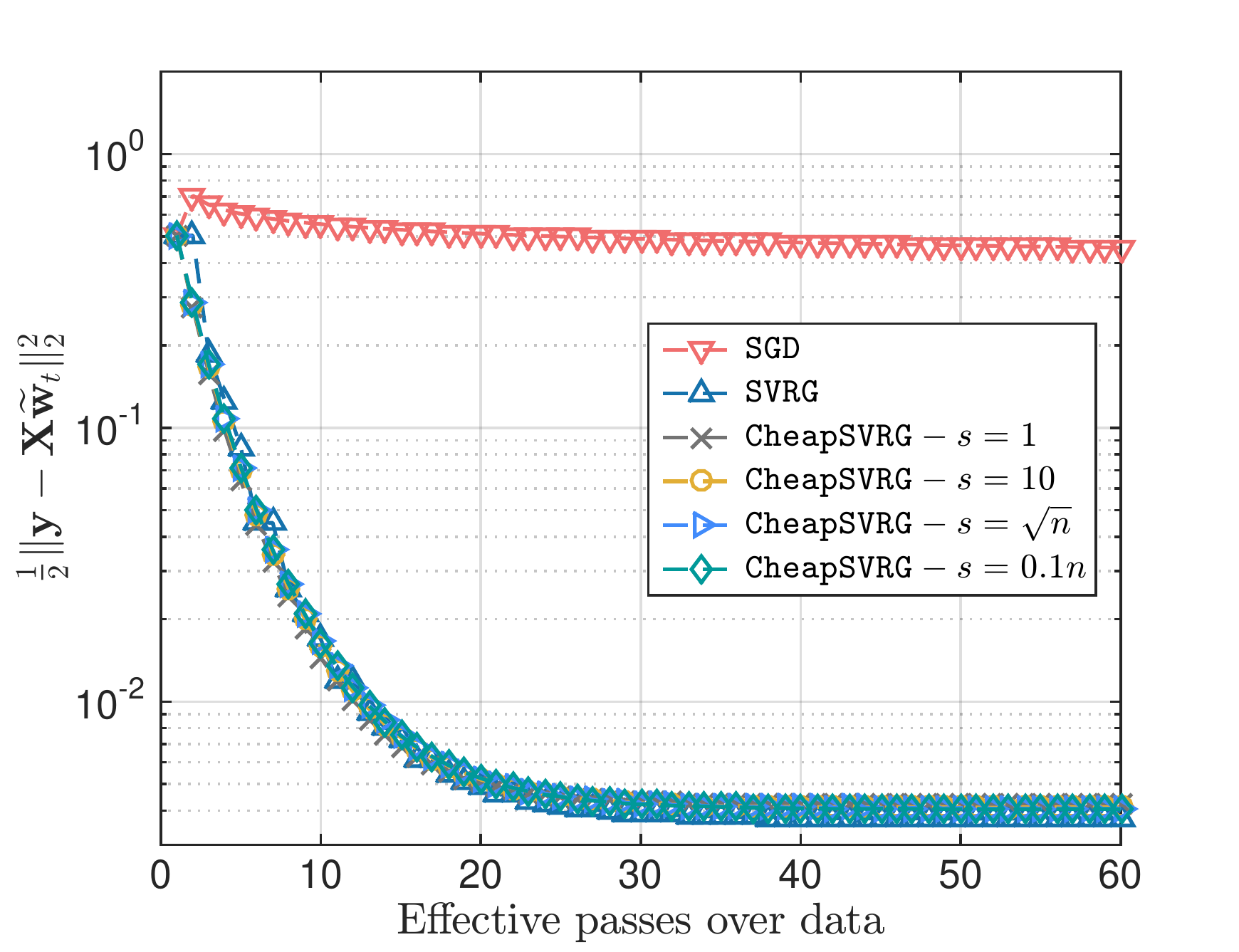} 
\includegraphics[width=0.32\textwidth]{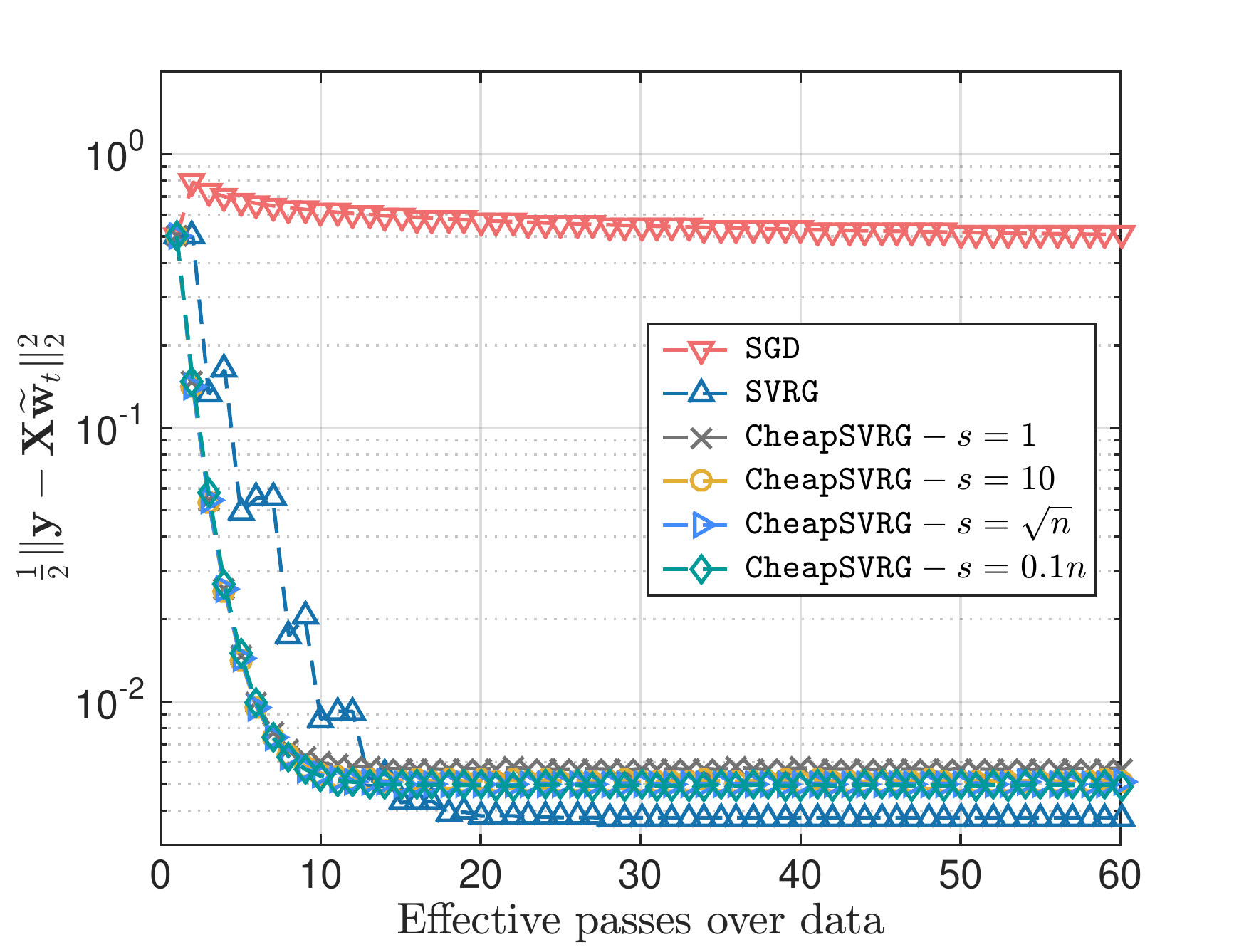}
\includegraphics[width=0.32\textwidth]{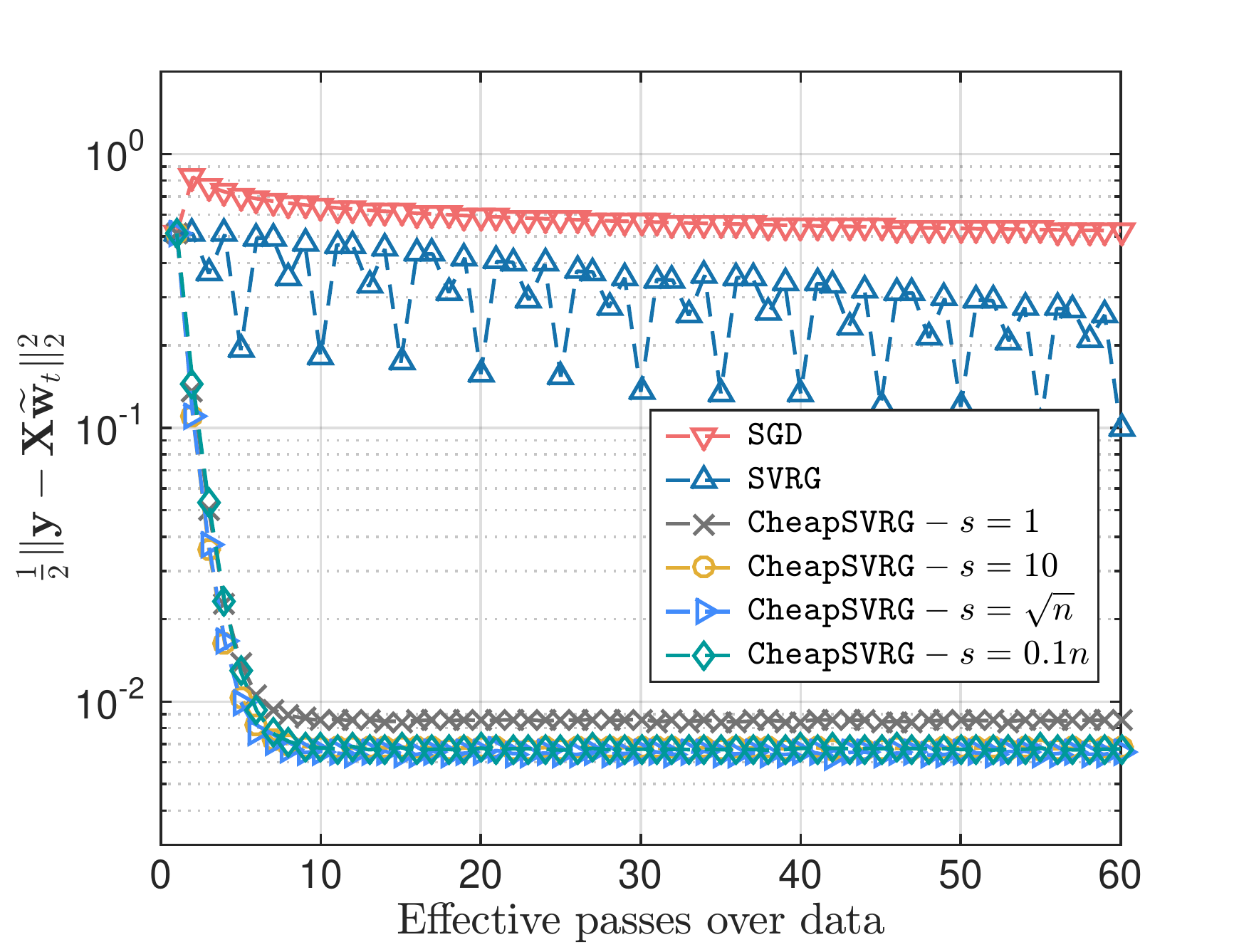}
\caption{Convergence performance w.r.t. $\sfrac{1}{2} \|\mathbf{y} - \mathbf{X}\wt_t\|_2^2$ vs the number of effective data passes -- 
\textit{i.e.}, the number of times $n$ data points were accessed --
for 
${\eta = ({100L})^{-1}}$ (left),
${\eta = ({300L})^{-1}}$ (middle), and
${\eta = ({500L})^{-1}}$ (right).
In all experiments, we generate noise such that $\|\boldsymbol{\varepsilon}\|_{2} = 0.1$.
The plotted curves depict the median over $50$ Monte Carlo iterations:
$10$ random independent instances of~\eqref{exp:eq_00},
$5$ executions/instance for each scheme. }
\label{fig:exp1}
\end{figure*}

\section{Experiments}
\label{sec:experiments}

We empirically evaluate \algo on synthetic and real data
and
compare mainly with \textsc{Svrg}~\cite{johnson2013accelerating}.
We show that in some cases it improves upon existing stochastic optimization methods, and discuss its properties, strengths and weaknesses.

\subsection{Properties of \algo}

We consider a synthetic linear regression problem:
given a set of training samples 
$(\x_1, y_1)$, $\dots$, $(\x_{\numsam}, y_{\numsam})$, 
where $\x_i \in \R^\dim$ and $y_i \in \R$, 
we seek 
the solution to
\begin{align}
	\min_{\w \in \R^\dim} 
	\frac{1}{\numsam} 
	\sum_{i = 1}^\numsam 
		\frac{\numsam}{2}
		\left(y_i - \x_i^\top \w\right)^2. 
		\label{exp:eq_00}
\end{align} 
We generate an instance of the problem as follows.
First, we randomly select a point ${\w^\star \in \mathbb{R}^{p}}$ from a spherical Gaussian distribution and rescale to unit $\ell_{2}$-norm;
this point serves as our `ground truth'.
Then, we randomly generate a sequence of $\x_i$'s i.i.d. according to a Gaussian $\mathcal{N}\left(0, \sfrac{1}{\numsam}\right)$ distribution.
Let $\mathbf{X}$ be the $p \times n$ matrix formed by stacking the samples $\x_{i}$, $i=1,\hdots, n$.
We compute
$\mathbf{y} = \mathbf{X} \w^\star + \boldsymbol{\varepsilon}$, where $\boldsymbol{\varepsilon} \in \R^\numsam$ is a noise term drawn from $\mathcal{N}\left(0, \mathbf{I}\right)$, with $\ell_{2}$-norm rescaled to a desired value controlling the noise level.


We set ${\numsam = 2\cdot 10^3}$ and ${\dim = 500}$.
Let $L = \sigma_{\max}^2(\mathbf{X})$ where $\sigma_{\max}$ denotes the maximum singular value of $\mathbf{X}$.
We run $(i)$ the classic \SGD method with decreasing step size $\eta_k \propto \sfrac{1}{k}$, $(ii)$ the \textsc{Svrg} method of Johnson and Zhang~\cite{johnson2013accelerating} and, $(iii)$ our \algo~for parameter values $s \in \left\{1,\, 10,\, \sqrt{\numsam},\, {0.1\numsam}\right\}$, which covers a wide spectrum of possible configurations for $s$.

\paragraph{Step size selection.}
We study the effect of the step size on the performance of the algorithms; see Figure~\ref{fig:exp1}. 
The horizontal axis represents the number \emph{effective passes} over the data:
evaluating~$\numsam$ component gradients, or computing a single full gradient is considered as \emph{one} effective pass.  
The vertical axis depicts the progress of the objective in~\eqref{exp:eq_00}. 
 
We plot the performance for three step sizes: 
$\eta = (cL)^{-1}$,
for $c=100, 300$ and $500$.
Observe that \textsc{Svrg} becomes slower if the step size is either too big or too small, as also reported in \cite{johnson2013accelerating, xiao2014proximal}.
The middle value $\eta = ({300L})^{-1}$ was the best\footnote{Determined via binary search.} for \textsc{Svrg} 
in the range we considered.
Note that each algorithm achieves its peak performance for a different value of the step size.
In subsequent experiments, however, we will use the above value which was best for \textsc{Svrg}.

Overall, we observed that \algo is more `flexible' in the choice of the step size.
In Figure~\ref{fig:exp1} (right), 
with a suboptimal choice of step size, \textsc{Svrg} oscillates and progresses slowly.
On the contrary, \algo converges nice even for ${s=1}$. 
It is also worth noting \algo with $s=1$,
\textit{i.e.}, effectively combining two datapoints in each stochastic update, achieves a substantial improvement compared to vanilla \SGD.

\begin{figure*}[t!]
\centering
\includegraphics[width=0.32\textwidth]{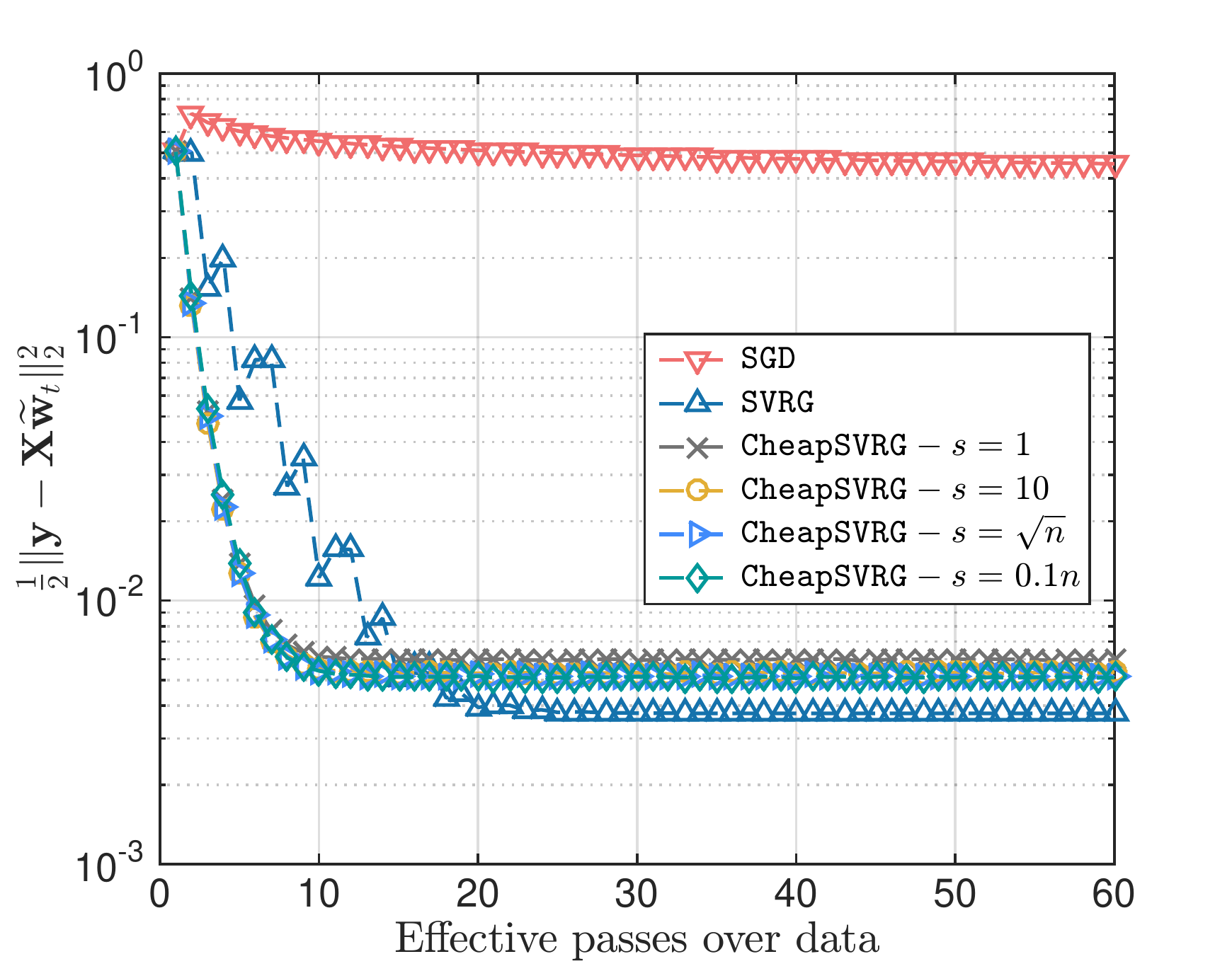} 
\includegraphics[width=0.32\textwidth]{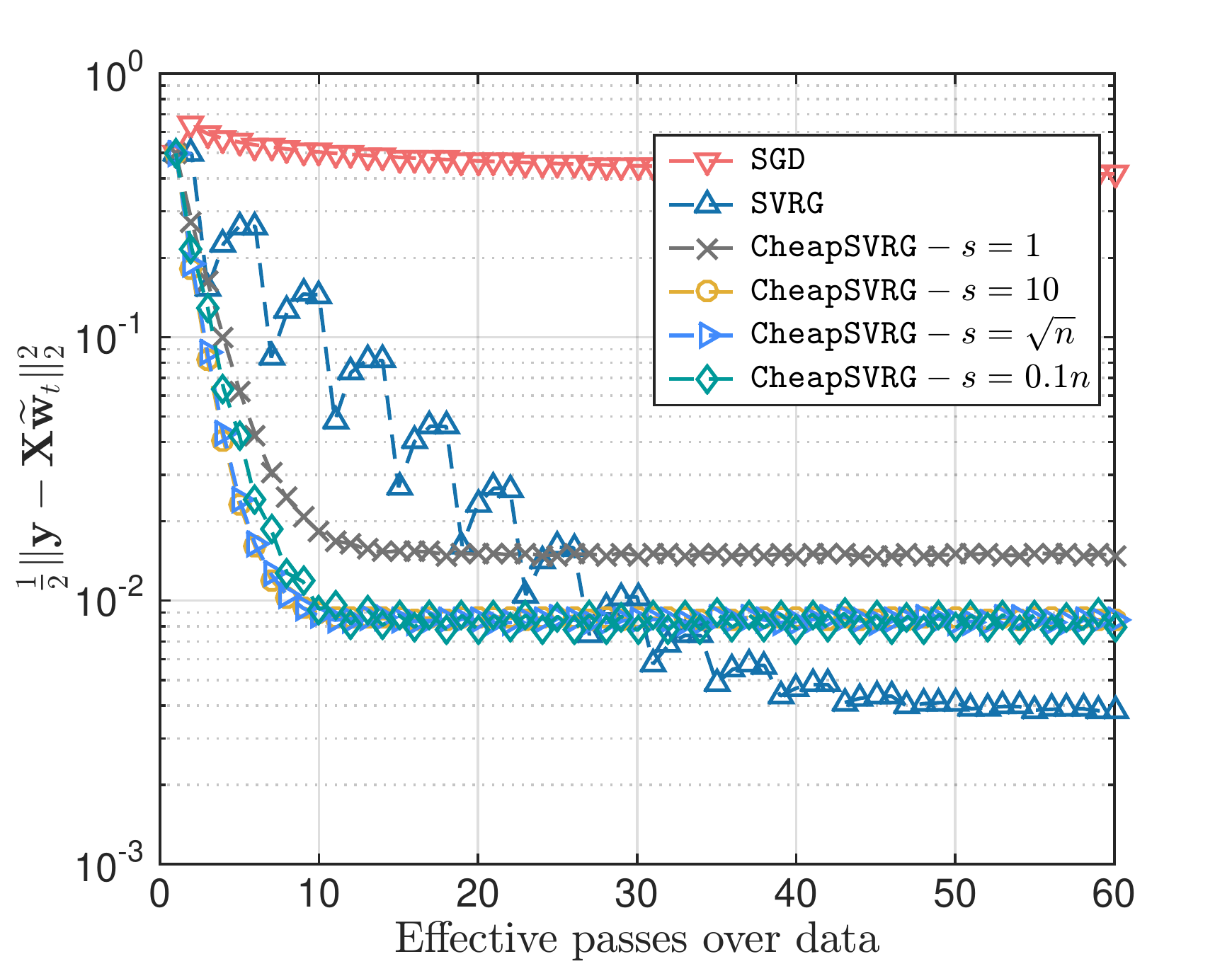}
\includegraphics[width=0.32\textwidth]{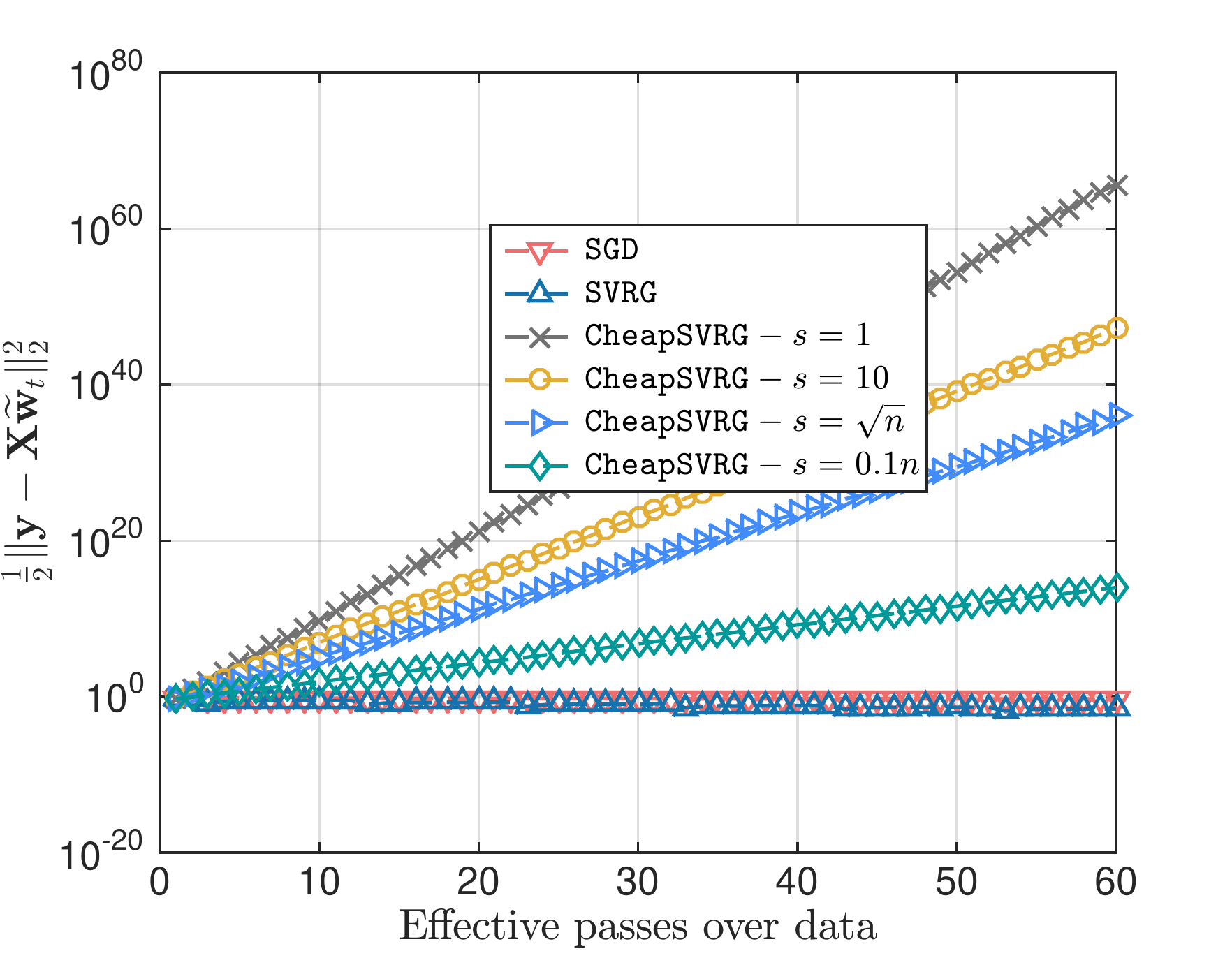}
\caption{Convergence performance w.r.t. $\sfrac{1}{2} \|\mathbf{y} - \mathbf{X}\wt_t\|_2^2$ vs. effective number of passes over the data. We set an upper bound on total atomic gradient calculations spent as $\nabla_{\text{total}} = 60\numsam = 12\cdot 10^4$ and vary the percentage of these resources in the inner loop two-stage \SGD schemes. Left: \texttt{perc} $= 60$\%. Middle : \texttt{perc} $= 75$\%. Right: \texttt{perc} $= 90$\%. In all experiments, we set $\|\boldsymbol{\varepsilon}\|_2 = 0.1$.
The plotted curves depict the median over $50$ Monte Carlo iterations:
$10$ random independent instances of~\eqref{exp:eq_00},
$5$ executions/instance for each scheme.}
\label{fig:exp3}
\end{figure*}

\paragraph{Resilience to noise.}
We study the behavior of the algorithms with respect to the noise magnitude.
We consider the cases
 $\|\boldsymbol{\varepsilon}\|_{2} \in \left\{0,  0.5, 10^{-2}, 10^{-1}\right\}$.
In Figure \ref{fig:exp2},
we focus on four distinct noise levels and plot the distance of the estimate from the ground truth $\w^{\star}$ vs the number of effective data passes.  
For \SGD, we use the sequence of step sizes $\eta_k = {{0.1} \cdot {L^{-1} \cdot k^{-1}}}$.

We also note the following surprising result: in the noiseless case, it appears that $s = 1$ is sufficient for linear convergence in practice; see Figure \ref{fig:exp2}.
In contrast, \algo~is less resilient to noise than \textsc{Svrg} -- however, we can still get to a good solution with less computational complexity per iteration.

\begin{figure*}[t!]
\centering
\includegraphics[width=0.24\textwidth]{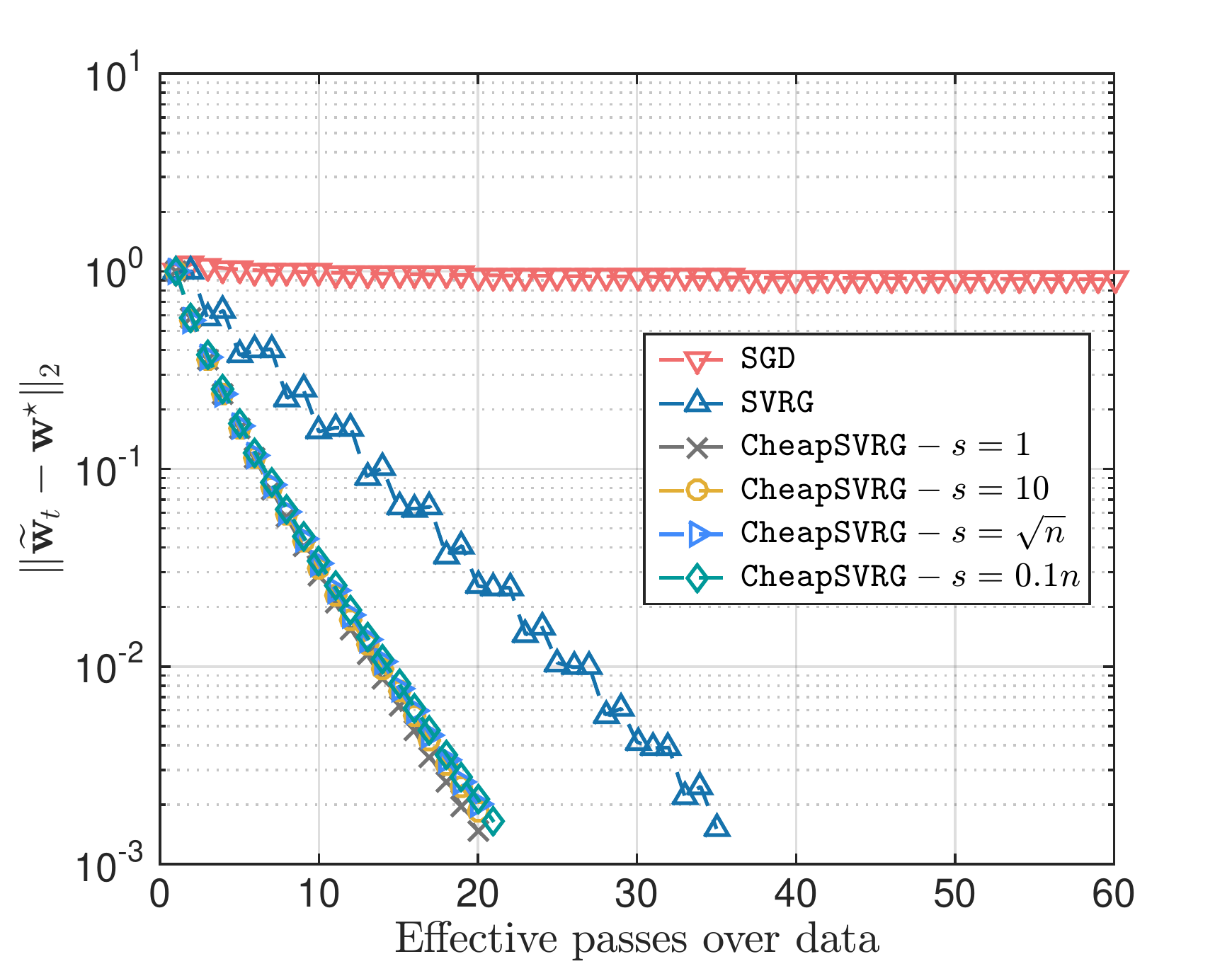} 
\includegraphics[width=0.24\textwidth]{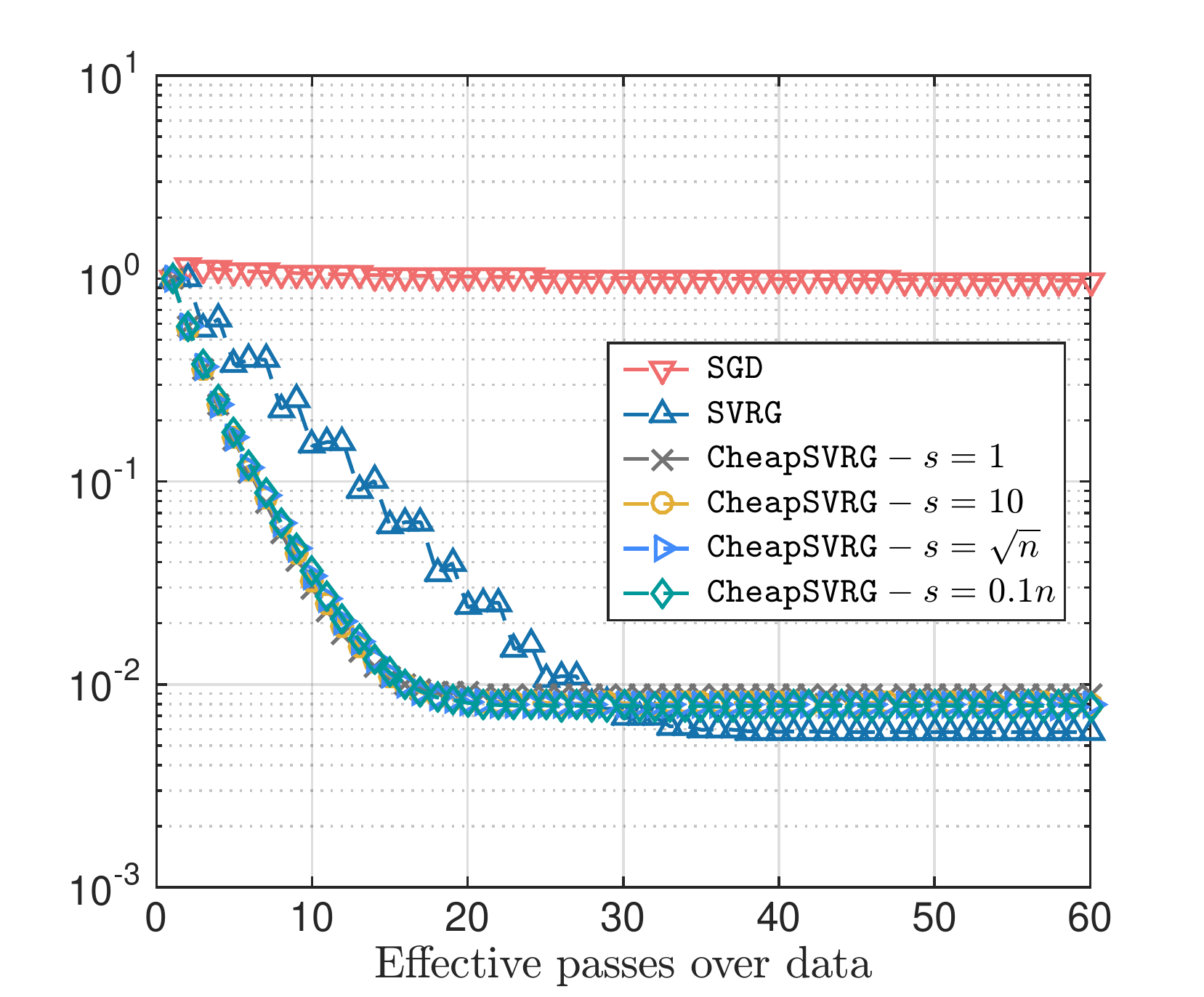} 
\includegraphics[width=0.24\textwidth]{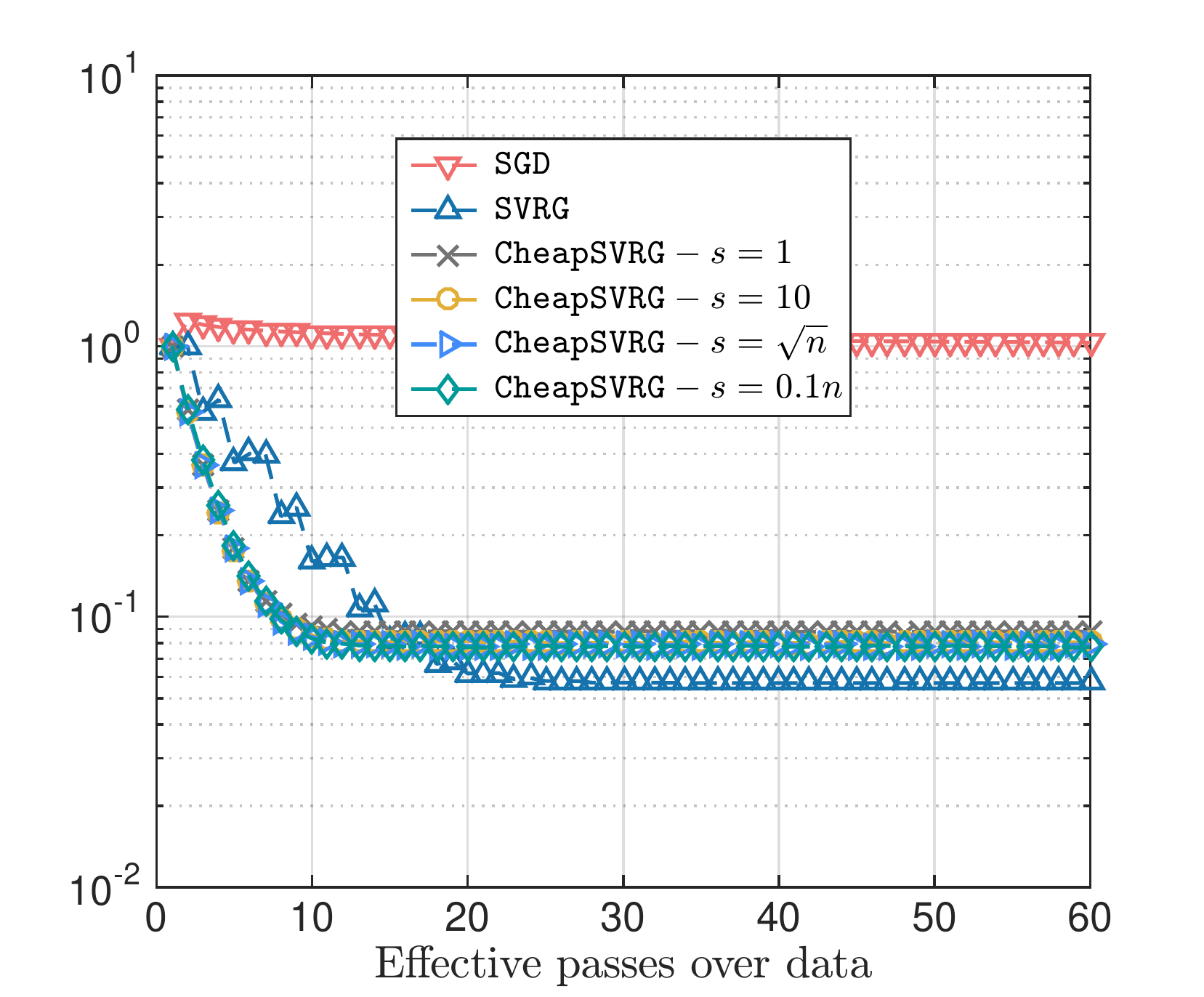} 
\includegraphics[width=0.24\textwidth]{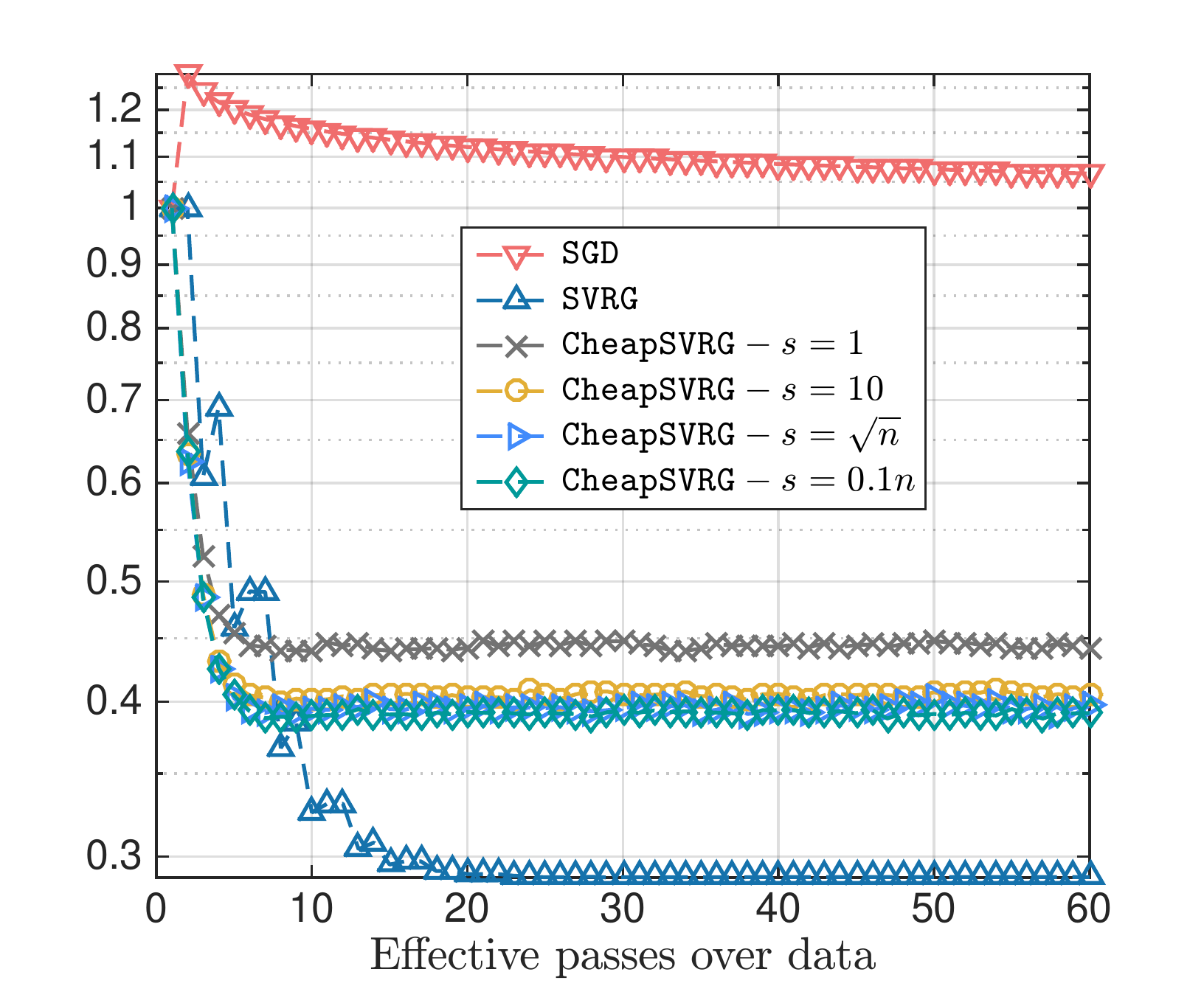}
\caption{Distance from the optimum vs the number of effective data passes for the linear regression problem.
We generate 10 independent random instances of~\eqref{exp:eq_00}. From left to right, we use noise
noise $\boldsymbol{\varepsilon}$ with standard deviation
$\|\boldsymbol{\varepsilon}\|_2 = 0$ (noiseless), 
$\|\boldsymbol{\varepsilon}\|_2 = 10^{-2}$, 
$\|\boldsymbol{\varepsilon}\|_2 = 10^{-1}$,
and $\|\boldsymbol{\varepsilon}\|_2 = 0.5$.
Each scheme is executed $5$ times/instance.
We plot the median over the $50$ Monte Carlo iterations.
}
\label{fig:exp2}
\end{figure*}

\paragraph{Number of inner loop iterations.}
Let $\nabla_{\text{total}}$ denote a budget of atomic gradient computations.
We study how the objective value decreases with respect to percentage \texttt{perc} of the budget allocated to the inner loop.
We first run a classic gradient descent with step size $\eta = \sfrac{1}{L}$ 
which converges within $\sim 60$ iterations. 
Based on this, we choose our global budget to be $\nabla_{\text{total}} = 60\numsam = 12\cdot 10^4$. 
We consider the following values for \texttt{perc}: $60\%, 75\%, 90\%$. \textit{E.g.}, when \texttt{perc} $= 90$\%, only $12000$ atomic gradient calculations are spent in outer loop iterations.
The results are depicted in Fig.~\ref{fig:exp3}.

We observe that convergence is slower as fewer computations are spent in outer iterations.
Also, in contrast to \textsc{Svrg},
our algorithm appears to be sensitive to the choice of \texttt{perc}:
for ${\texttt{perc} = 90\%}$, our scheme diverges, while \textsc{Svrg} finds relatively good solution. 

\subsection{$\ell_2$-regularized logistic regression}
We consider the regularized logistic regression problem,
\textit{i.e.}, the minimization
\begin{align}
	\min_{\w \in \R^\dim} \frac{1}{\numsam} \sum_{i = 1}^\numsam \log\left(1 + e^{-y_i \cdot \x_i^\top \w}\right) + \lambda \cdot \|\w\|_2^2.
\end{align} 
Here, $\left(y_i, \x_i\right) \in \left\{-1, 1\right\} \times \R^\dim$, where $\y_i$ indicates the binary label in a classification problem, $\w$ represents the predictor, and $\lambda > 0$ is a regularization parameter.

We focus on the training loss in such a task.\footnote{By \cite{johnson2013accelerating}, we already know that such two-stage \SGD~schemes perform better than vanilla \SGD.}
We use the real datasets listed in Table~\ref{table:datasets:info}. 
We pre-process the data so that $\|\x_i\|_2 = 1, \forall i$, as in~\cite{xiao2014proximal}. This leads to an upper bound on Lipschitz constants for each $f_i$ such that $L_i \leq L := \sfrac{1}{4}$.
We set ${\eta = {0.1}/{L}}$ for all algorithms under consideration, according to \cite{johnson2013accelerating, xiao2014proximal}, \texttt{perc} $= 75$\% and, $\lambda = 10^{-6}$ for all problem cases.

Fig.~\ref{fig:exp4} depicts the convergence results for the   \texttt{marti0}, \texttt{reged0} and \texttt{sido0} datasets. 
\algo achieves comparable performance to \textsc{Svrg}, while requiring less computational `effort' per epoch: though smaller values of $s$, such that $s = 1$ or $s = 10$, lead to slower convergence, \algo still performs steps towards the solution, while the complexity per epoch is significantly diminished.

\begin{figure*}[t!]
\centering
\includegraphics[width=0.32\textwidth]{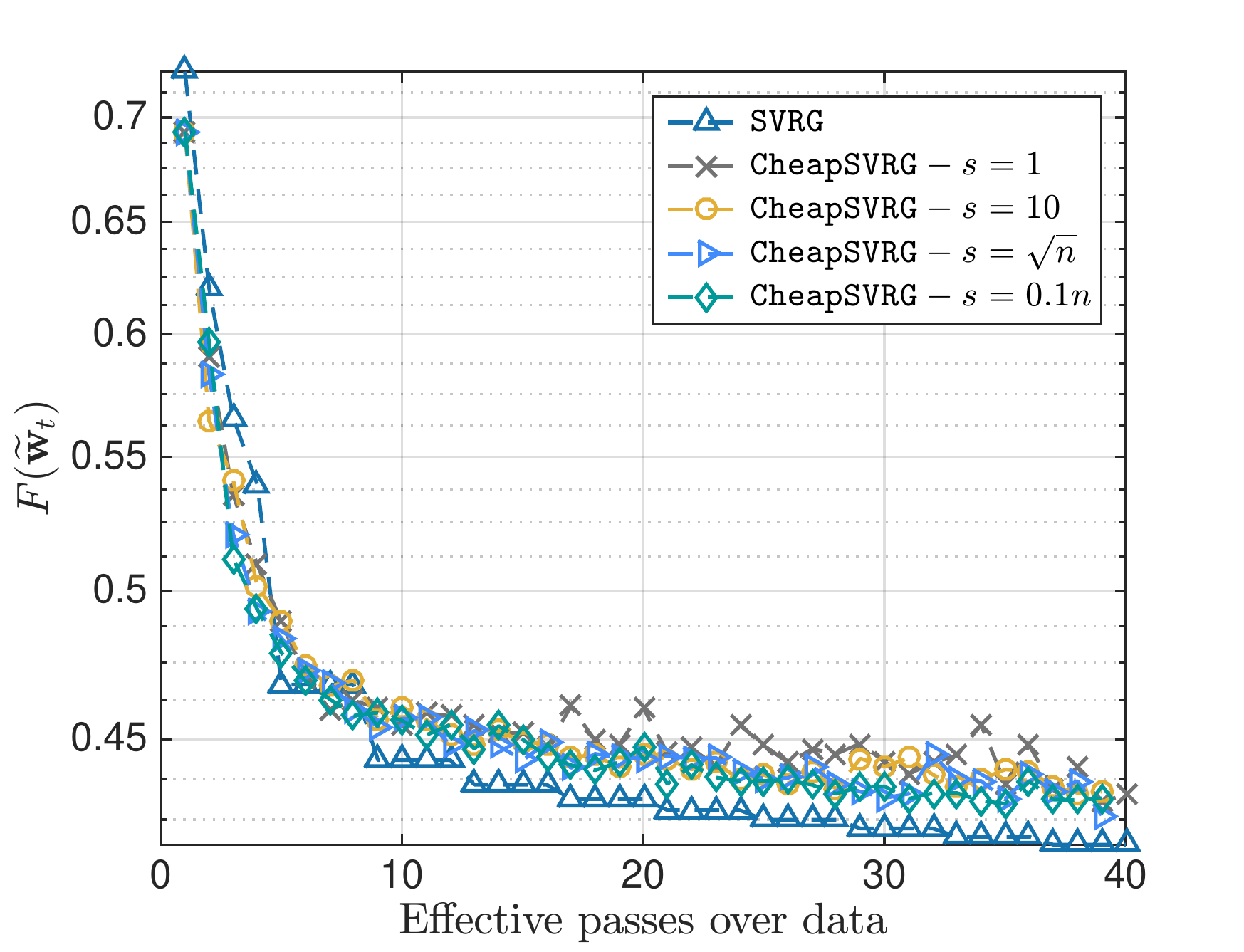} 
\includegraphics[width=0.32\textwidth]{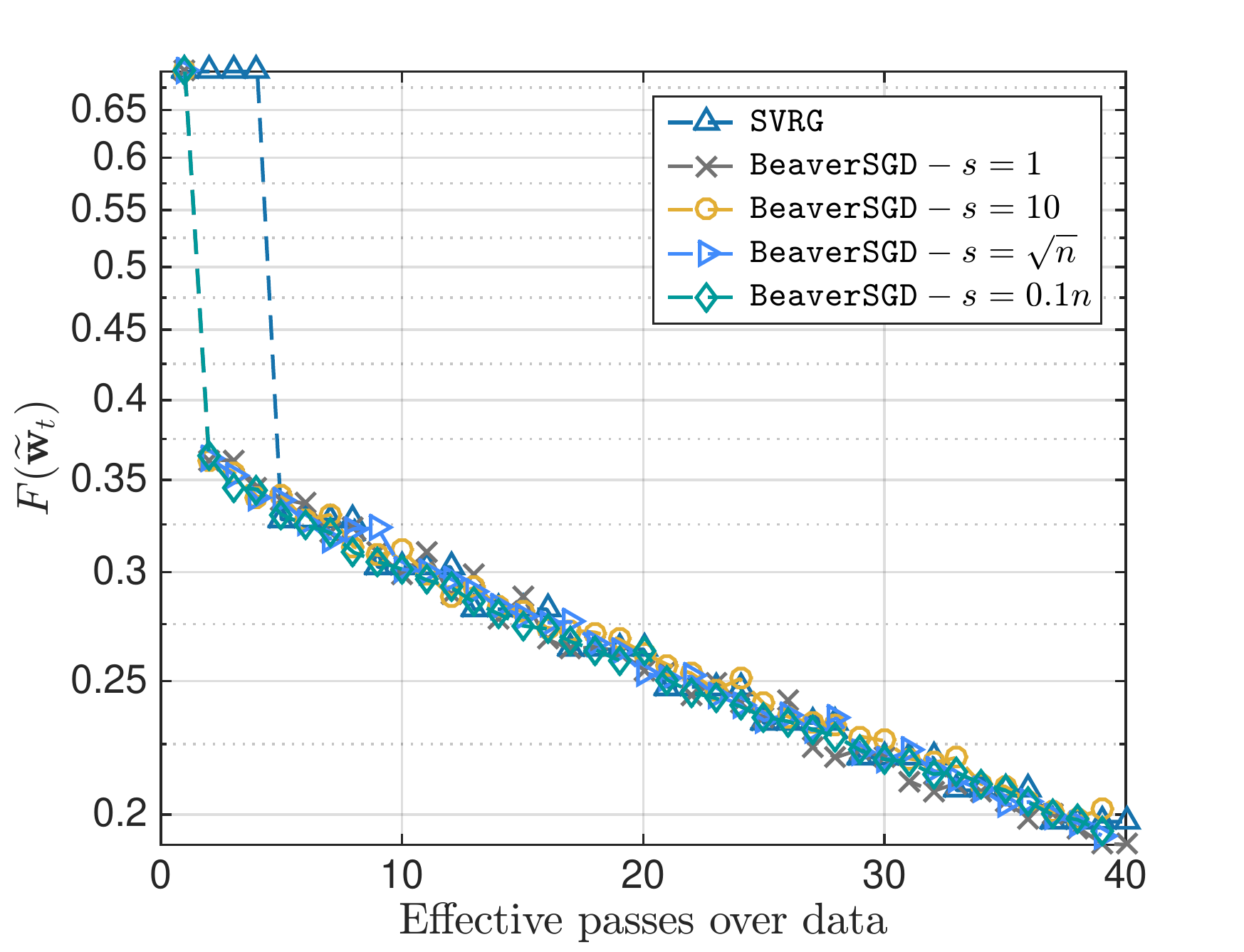}
\includegraphics[width=0.32\textwidth]{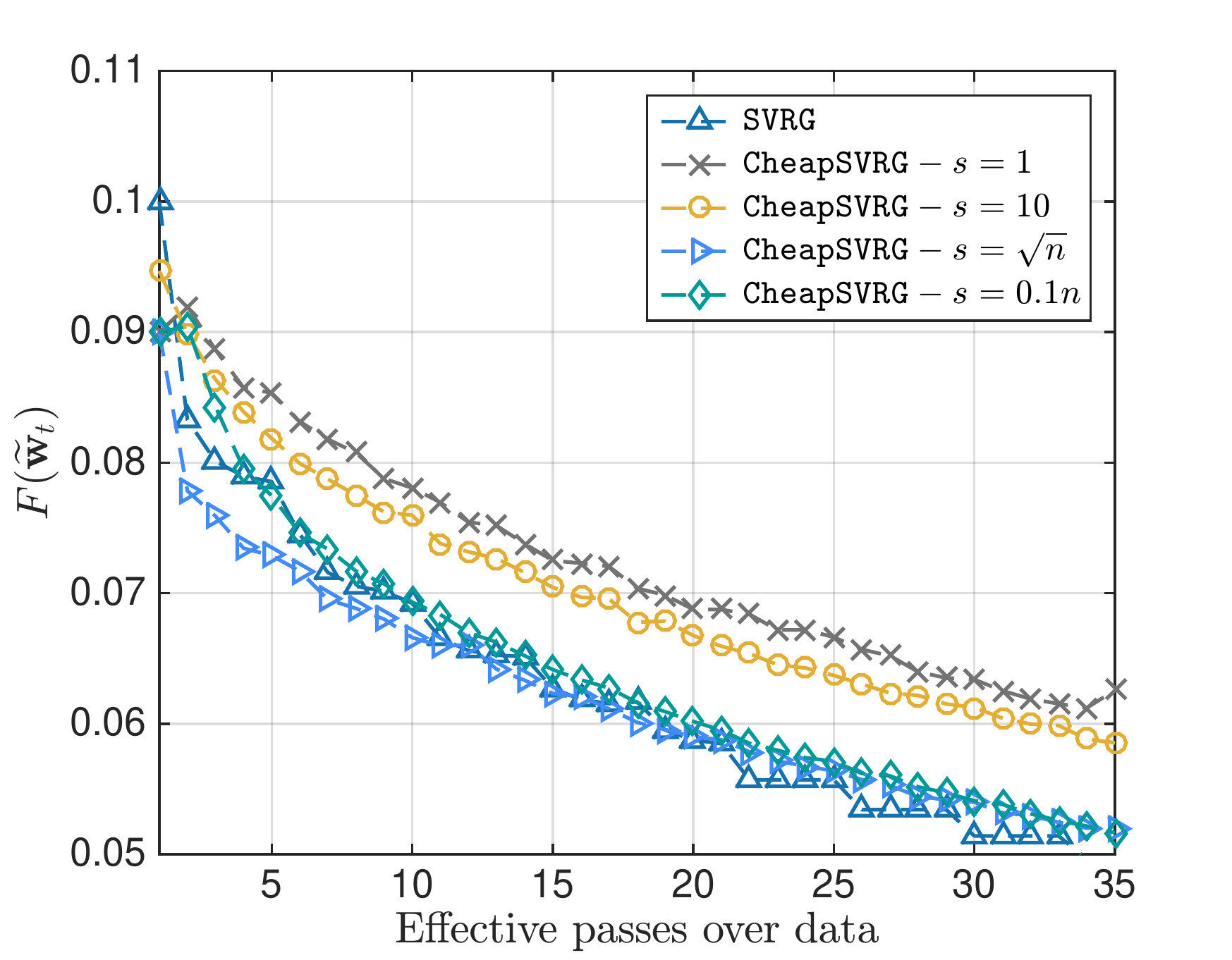}
\caption{Convergence performance of algorithms for the $\ell_2$-regularized logistic regression objective. From left to right, we used the \texttt{marti0}, \texttt{reged0}, and \texttt{sido0} dataset; the description of the datasets is given in Table \ref{table:datasets:info}. Plots depict $F(\wt_t)$ vs the number of effective data passes.
We use step size $\eta = 0.1/L$ for all algorithms, as suggested by \cite{johnson2013accelerating, xiao2014proximal}. 
The curves depict the median over $10$ Monte Carlo iterations.}
\label{fig:exp6}
\end{figure*}

\begin{table}[t!]
\centering
\footnotesize
\rowcolors{2}{white}{black!05!white}
\begin{tabular}{c c c c c c c}
  \toprule
  Dataset & & $\numsam$ & & $\dim$ \\ 
  \cmidrule{1-1} \cmidrule{3-3} \cmidrule{5-5}
   \texttt{marti0} & & $1024$ & & $500$ \\
   \texttt{reged0} & & $999$ & & $500$ \\
   \texttt{sido0} & & $12678$ & & $4932$ \\
  \bottomrule
\end{tabular}
\caption{Summary of datasets \cite{guyon2008dataset}.}
\label{table:datasets:info}
\end{table}

\section{Conclusions}

We proposed \algo, a new variance-reduction scheme for stochastic optimization, based on~\cite{johnson2013accelerating}.
The main difference is that 
instead of computing a full gradient in each epoch,
our scheme computes a surrogate utilizing only part of the data, thus, reducing the per-epoch complexity.
\algo comes with convergence guarantees:
under assumptions, it achieves a linear convergence rate up to some constant neighborhood of the optimal.
We empirically evaluated our method and discussed its strengths and weaknesses.

There are several future directions.
In the theory front, it would be interesting to maintain similar convergence guarantees under fewer assumptions,
extend our results beyond the smooth convex optimization, \textit{e.g.}, to the proximal setting, 
or develop distributed variants.
Finally, we seek to apply our \algo to large-scale problems, \textit{e.g.}, for training large neural networks.
We hope that this will help us better understand the properties of \algo and the trade-offs associated with its various configuration parameters.

%

\FloatBarrier
\begin{small}
	\bibliography{CheapSVRG}
	\bibliographystyle{plain}
\end{small}


\newpage
\clearpage
\appendix
%
%
%

\section{Proof of Theorem \ref{theorem:complexity}}
By assumptions of the theorem, we have:
\begin{small}
\begin{align*}
\eta < \dfrac{1}{4L\left((1 + \theta)+ \sfrac{1}{s}\right)} ~\text{and}~ K > \frac{1}{(1 - \theta) \eta \left(1 - 4 L \eta\right) \gamma},
\end{align*}
\end{small} As mentioned in the remarks of Section \ref{sec:analysis}, the above conditions are sufficient to guarantee $\rho < 1$, for some $\theta \in (0,1)$. Further, for given accuracy parameter $\epsilon$, we assume $\kappa \leq \frac{\epsilon}{2}$. 

Let us define 
\begin{align*}
	\varphi_t :=
	\mathbb{E}\mathopen{}
	\bigl[
		F(\wt_t ) - F(\ws)
	\bigr],
\end{align*} as in the proof of Theorem \ref{theorem:convergence}. In order to satisfy $\varphi_T \leq \epsilon$, it is sufficient to find the required number of iterations such that $\rho^T \varphi_0 \leq \frac{\epsilon}{2}$. In particular:

\begin{small}\begin{align*}
\rho^T \varphi_{0} \leq \frac{\epsilon}{2} &\Rightarrow -\left(T \log{\rho} + \log \varphi_0\right) \geq - \log \frac{\epsilon}{2} \\
														 &\Rightarrow T \cdot \log \left(\rho^{-1}\right) \geq \log \frac{2}{\epsilon} + \log \varphi_0 \\
														 &\Rightarrow T \geq \left(\log \frac{1}{\rho}\right)^{-1} \cdot \log \left(\frac{2\left(F(\wt_0) - F(\ws)\right)}{\epsilon}\right)
\end{align*}\end{small}
Moreover, each epoch involves $K$ iterations in the inner loop. Each inner loop iteration involves two atomic gradient calculations. Combining the above, we conclude that the total number of gradient computations required to ensure that  $\varphi_T \leq \epsilon$ is $O\left((2K + s)\log{\sfrac{1}{\epsilon}}\right)$.

\section{Mini-batches in \textsc{CheapSVRG}}{\label{sec:mini_batch}}

In the sequel, we show how Alg.~\ref{algo:beaverSGD} can also accommodate mini-batches in the inner loops and maintain similar convergence guarantees, under Assumptions~\ref{ass:00}-\ref{ass:03}. The resulting algorithm is described in Alg. \ref{algo:MBbeaverSGD}. In particular:

\begin{algorithm}[!ht]
	\caption{\algo with mini batches}
	\label{algo:MBbeaverSGD}
	{
	\begin{algorithmic}[1]
	   	\STATE \textbf{Input}: $\widetilde{\w}_0, \eta, s, q, K, T$.
	   	\STATE \textbf{Output}: $\wt_T$.
	   	\FOR{ $t = 1, 2, \dots, T$ }
	   		\STATE Randomly select $\S_{t} \subset [\numsam]$ with cardinality $s$.
	   		\STATE Set $\widetilde{\w} = \widetilde{\w}_{t-1}$ and $\S = \S_{t}$.
				\STATE $\widetilde{\boldsymbol{\mu}}_\S = \sfrac{1}{s} \sum_{i \in \S} \gradf_i(\wt)$.
				\STATE $\w_0 = \widetilde{\w}$.
				\FOR{ $k = 1, \dots, K-1$ }
				    \STATE Randomly select $\Q_k \subset [\numsam]$ with cardinality $q$.
					\STATE Set $\Q = \Q_k$.
					\STATE $\v_k = \nabla f_{\Q}(\w_{k-1}) - \nabla f_{\Q} (\widetilde{\w}) + \widetilde{\boldsymbol{\mu}}_\S$.
					\STATE $\w_{k} = \w_{k-1} - \eta \cdot \v_k$.
				\ENDFOR
				\STATE $\widetilde{\w}_t = \frac{1}{K} \sum_{j = 0}^{K-1} \w_{j}$.
			\ENDFOR
	\end{algorithmic}
	}
\end{algorithm}

\begin{theorem}[Iteration invariant]
\label{with_q:theorem:convergence}
Let $\ws$ be the optimal solution for minimization~\eqref{intro:eq_00}.
Further, let $s$, $q$, $\eta$, $T$ and $K$ be user defined parameters such that
\begin{align*}
	\rho \eqdef 
	\frac{q}{\eta \cdot \left(q - {4 L \cdot \eta}\right) \cdot {K \cdot \gamma}} + \frac{4L \cdot \eta \cdot\left(s + q\right)}{\left(q - 4 L \cdot \eta\right) \cdot s} 
	< 1.
\end{align*}
Under Asm.~\ref{ass:00}-\ref{ass:03}, \algo~satisfies the following:
\begin{align*}
	&
	\mathbb{E}\bigl[ F(\wt_T) - F(\ws) \bigr]
	\leq \rho^{T} \cdot \left(F(\wt_0)-F(\ws)\right) +\frac{q}{q - 4 L \eta} \cdot \left(\frac{2\eta}{s} + \frac{\zeta}{K}\right) \cdot \max\left\{\xi, \xi^2\right\} \cdot \frac{1}{1 - \rho}.
\end{align*}
\end{theorem}

\subsection{Proof of Theorem \ref{with_q:theorem:convergence}}
To prove~\ref{with_q:theorem:convergence},
we analyze \algo starting from its core inner loop (Lines $9$-$14$).
We consider a fixed subset $\S \subseteq [n]$ and show that in expectation, the steps of the inner loop make progress towards the optimum point.
Then, we move outwords to the `wrapping' loop that defines consecutive epochs
to incorporate the randomness in selecting the set $\S$.

We consider the $k$th iteration of the inner loop, during the $t$th iteration of the outer loop;
we consider a fixed set $\S \subseteq [n]$, starting point $\w_0 \in \mathbb{R}^{p}$ and (partial) gradient information $\mt_\S \in \mathbb{R}^{p}$  as defined in Steps $6-8$ of Alg.~\ref{algo:beaverSGD}.
The set~$\Q_k$ is randomly selected from $[\numsam]$ with cardinality $|\Q_k| = q$.
Similarly to the proof of Thm.~\ref{theorem:convergence}, we have:
\begin{align}
	&\!\!
	\mathbb{E}_{\Q_k}\mathopen{}
	\left[
	\left\|
		\w_k -\ws
	\right\|_{2}^{2}
	\right]
	= 
	\left\| \w_{k-1}-\ws \right\|_{2}^{2} - 2\eta \cdot \left(\w_{k-1} - \ws \right)^\top 
	\mathbb{E}_{\Q_k}\left[\v_k\right] + \eta^2 
	\mathbb{E}_{\Q_k}\mathopen{}\left[\norm{\v_k}_2^2 \right].
 \label{with_q:eq:001}
\end{align} 
where the expectation is with respect to the random variable $\Q_k$.
By the definition of $\v_k$ in Line 12,
\begin{align}
	\mathbb{E}_{\Q_k}\mathopen{}
	\left[
		\v_{k}
	\right]
	&=
	\mathbb{E}_{\Q_k}\mathopen{}
	\left[
		\nabla f_{\Q_k}(\w_{k-1}) - \nabla f_{\Q_k}(\wt) + \mt_\S
	\right] \nonumber\\
	&=
	\nabla F(\w_{k-1}) - \nabla F(\wt) + \mt_\S,
\end{align}
where the second step follows from the fact that $\Q_{k}$ is selected uniformly at random from $[n]$.
Similarly,
\begin{align} 
	\mathbb{E}_{\Q_k}\mathopen{}
	\left[
		\left\| \v_{k} \right\|_{2}^{2}
	\right]
	&=
	\mathbb{E}_{\Q_k}\mathopen{}
	\left[
		\left\|
			\nabla f_{\Q_k}(\w_{k-1}) - \nabla f_{\Q_k}(\wt) + \mt_\S
		\right\|_{2}^{2}
	\right]
	\nonumber \\ 
	&\;\;\stackrel{(i)}{\leq} 
	4\cdot 
	\mathbb{E}_{\Q_k}\mathopen{}
	\left[
		\|\nabla f_{\Q_k}(\w_{k-1}) - \nabla f_{\Q_k}(\ws)\|_2^2
	\right] +
		4\cdot
		\mathbb{E}_{\Q_k}\mathopen{}
		\left[
			\left\|\nabla f_{\Q_k}(\wt) - \nabla f_{\Q_k}(\ws)\right\|_{2}^{2}
		\right]
		+2 \cdot \|\mt_\S\|_2^2 
		\nonumber\\
	&\;\;\stackrel{(ii)}{\leq}
	\sfrac{8L}{q} \cdot
	\left(
		F(\w_{k-1}) - F(\ws) + F(\wt) - F(\ws)
	\right) + 2\cdot \|\mt_\S\|_{2}^{2}.
	\label{with_q:bound_on_expected_norm_of_v_k}
\end{align}
Inequality $(i)$ is follows by applying $\norm{\x - \y}_2^2 \leq 2\norm{\x}_2^2 + 2\norm{\y}_2^2$ on all atomic gradients indexed by $\Q_k$,
while $(ii)$ is due to the following lemma.

\begin{lemma}
Given putative solution $\w_{k-1}$ and mini-batch $\Q_k$ with cardinality $|\Q_k| = q$, the following holds true on expectation:
\begin{align*}
\Ex{\Q_k}{\|\nabla f_{\Q_k}(\w_{k-1}) - \nabla f_{\Q_k}(\ws)\|_2^2}  \leq
	q^{-1} \cdot {2L} \cdot \left(F(\w_{k-1}) - F(\ws)\right).
\end{align*}
\end{lemma}
\begin{proof}
let 
$$
	\mathfrak{Q}
	=
	\left\{\Q^i~:~ |\Q^i| = |\Q|, ~\Q^i \subset [n], ~\Q^i \neq \Q^j, \forall i \neq j \right\},
$$
\textit{i.e.}, $\mathfrak{Q}$ contains all different index sets of cardinality $|\Q| = q$. Observe that $|\mathfrak{Q}| = {n \choose |\Q|}$.Note that the set $\Q_k$ randomly selected in the inner loop of Alg.~\ref{algo:beaverSGD} is a member of $\mathfrak{Q}$.
Then:
\begin{align}
	&
	\Ex{\Q_k}{\|\nabla f_{\Q_k}(\w_{k-1}) - \nabla f_{\Q_k}(\ws)\|_2^2} \nonumber\\
	&=
	\mathbb{E}
		\Bigl[ 			
			\bigl\||\Q_k|^{-1} \cdot
				  \sum_{i \in \Q_k} \left(\nabla f_i(\w_{k-1}) - \nabla f_i(\ws)\right)
			\bigr\|_{2}^{2}
		\Bigr] \nonumber\\
	&=
	\sum_{\Q^j \in \mathfrak{Q}} 
	\mathbb{P}[\Q^j] \cdot 
	\Bigl\|
		q^{-1} \cdot \sum_{i \in \Q^j} (\gradf_i(\w_{k-1}) - \gradf_i(\ws))
	\Bigr\|_2^2 \nonumber\\
 	&\stackrel{(i)}{=}
 	\sum_{\Q^j \in \mathfrak{Q}}
 		{n \choose q}^{-1} \cdot 
 		\Bigl\|
 			q^{-1} \cdot \sum_{i \in \Q^j} (\gradf_i(\w_{k-1}) - \gradf_i(\ws))
 		\Bigr\|_{2}^{2} \nonumber\\
 	&\stackrel{(ii)}{=}
 	{n \choose q}^{-1} \cdot 
 	q^{-2} \cdot  \sum_{\Q^j \in \mathfrak{Q}} 
 	\Bigl\|
 		\sum_{i \in \Q^j} (\gradf_i(\w_{k-1}) - \gradf_i(\ws))
 	\Bigr\|_{2}^{2} \nonumber\\
 	&\leq
 	{n \choose q}^{-1} \cdot 
 	q^{-2} \cdot  \sum_{\Q^j \in \mathfrak{Q}} \sum_{i \in \Q^j} 
 	\bigl\|
 		(\gradf_i(\w_{k-1}) - \gradf_i(\ws))
 	\bigr\|_{2}^{2}\nonumber.				
\end{align}
Here, equality $(i)$ is follows from the fact that each $\Q^{j}$ is selected equiprobably from the set $\mathfrak{Q}$.
Equality $(ii)$ is due to the fact that $|\Q^j| = q,~\forall i$.

Since each set in $\mathfrak{Q}$ has cardinality $q$,
each data sample $i \in [n]$ contributes exactly ${n - 1 \choose q - 1}$ summands in the double summation above. 
This further leads to:
\begin{align*}
	&\Ex{\Q_k}{\|\nabla f_{\Q_k}(\w_{k-1}) - \nabla f_{\Q_k}(\ws)\|_2^2}  \nonumber\\
	&\leq 
	{n \choose q}^{-1} \cdot \frac{1}{q^2} \cdot {n-1 \choose q-1} \sum_{i=1}^n \norm{\gradf_i(\w_{k-1}) - \gradf_i(\ws)}_2^2 \\
	&\leq
	q^{-1} \cdot n^{-1}  \cdot \sum_{i=1}^n \norm{\gradf_i(\w_{k-1}) - \gradf_i(\ws)}_2^2\\
	&\stackrel{(i)}{\leq}
	q^{-1} \cdot {2L} \cdot \left(F(\w_{k-1}) - F(\ws)\right).
\end{align*}
Inequality $(i)$ follows from the the fact that
for any $\w \in \mathbb{R}^{p}$ \cite{johnson2013accelerating},
\begin{align}
	\frac{1}{\numsam} \cdot 
	\sum_{i=1}^\numsam 
		\left\| \gradf_i(\w) -\gradf_i(\ws) \right\|_{2}^{2} 
	\leq 
	2L \cdot \left(F(\w) - F(\ws)\right).
\label{with_q:eq:002}
\end{align} 
Similarly, for $\wt$,
\begin{align}
	\mathbb{E}_{\Q_k}\mathopen{}
	\left[
		\left\|\nabla f_{\Q_k}(\w_{k-1}) - \nabla f_{\Q_k}(\ws)\right\|_{2}^{2}
	\right]
	\leq 
	q^{-1} \cdot {2L} \cdot \left(F(\wt) - F(\ws)\right).
\end{align}
\end{proof}

Using the above lemma in~\eqref{with_q:bound_on_expected_norm_of_v_k}, the remainder of the proof follows that of Theorem \ref{theorem:convergence}.
\end{document}